\documentclass[twocolumn, journal]{IEEEtran}
\usepackage{graphicx}
\usepackage{epstopdf}
\usepackage{amsthm}
\theoremstyle{definition}
\usepackage{epsfig}
\usepackage{latexsym}
\usepackage{amsfonts}
\usepackage{here}
\usepackage{rawfonts}
\usepackage[utf8]{inputenc}
\usepackage[T1]{fontenc}
\usepackage{calc}
\usepackage{capitalgreekitalic}
\usepackage{url}
\usepackage{enumerate}
\usepackage{color}
\usepackage[tbtags]{amsmath}
\usepackage{amssymb}
\usepackage{upref}
\usepackage{epic,eepic}
\usepackage{times}
\usepackage{dsfont}
\usepackage{comment}
\usepackage{cite}
\usepackage{bbm}
\usepackage{bm}
\usepackage{amsmath}
\usepackage{cleveref}

\newcommand{\E}{\mathbb{E}}
\newcommand{\prob}{\mathbb{P}}

\allowdisplaybreaks[4]

\newcommand{\R}{\ensuremath{{\mathbb R}}}

\newcommand{\mysection}[1]{\vspace*{-0.35em}\section{#1}\vspace*{-0.35em}}
\newcommand{\mysubsection}[1]{\vspace*{-0.15em}\subsection{#1}\vspace*{-0.15em}}

\renewcommand{\Pr}{\ensuremath{\operatorname{Pr}}}

\newtheorem{theorem}{\bf Theorem}

\newtheorem{lemma}{\bf Lemma}

\newcommand{\assumptionitem}[2]{%
  \refstepcounter{assumption}%
  \item {\it Assumption~\theassumption\ (#1):}\label{#2}%
}
\newcommand{\assumpref}[1]{Assumption~\ref{#1}}

\usepackage{dsfont}
\newcounter{step}
\newlength{\totlinewidth}
  {\end{list}%
  \rule{\linewidth}{1pt}}
\newcounter{substep}

  {\end{list}}

\newlength{\aligntop}
\newlength{\alignbot}
 \makeatletter
\renewenvironment{align}{%
  \vspace{\aligntop}
  \start@align\@ne\st@rredfalse\m@ne
}{%
  \math@cr \black@\totwidth@
  \egroup
  \ifingather@
    \restorealignstate@
    \egroup
    \nonumber
    \ifnum0=`{\fi\iffalse}\fi
  \else
    $$%
  \fi
  \ignorespacesafterend%
  \vspace{\alignbot}\par\noindent
} \makeatother

\IEEEoverridecommandlockouts

\usepackage{algorithm}%,algpseudocode}
\usepackage{algorithmic}

\usepackage{subfigure}

\begin{document}
\pagestyle{empty}
\title{\huge A Query-Driven Communication-Efficient \\ Digital Twins Design for Autonomous Driving   
}
\author{
Nuocheng Yang, Longyu Zhou, Sihua Wang, Changchuan Yin, and {Tony Q.~S. Quek}, \emph{Fellow, IEEE}
\thanks{N. Yang, S. Wang, and C. Yin are with the Beijing Laboratory of Advanced Information Network, and the Beijing Key Laboratory of Network System Architecture and Convergence, Beijing University of Posts and Telecommunications, Beijing 100876, China (emails: \{yangnuocheng, sihuawang, ccyin\}@bupt.edu.cn).}
\thanks{L. Zhou and T. Q.~S. Quek are with the Information Systems Technology and Design Pillar, Singapore University of Technology and Design, 487372, Singapore (emails: zhoulyfuture@outlook.com, tonyquek@sutd.edu.sg).}
% \thanks{M. Chen is with the Department of Electrical and Computer Engineering and Institute for Data Science and Computing, University of Miami, Coral Gables, FL, 33146 USA (email: \protect\url{mingzhe.chen@miami.edu}).}
}
\maketitle

% placeholder — written last
\begin{abstract}
Digital twins (DTs) have become a potential technology to perform risk-free simulation of physical entities for deterministic and high-reliability services in diverse scenarios such as autonomous driving and low-altitude economy.
In the autonomous driving scenario, traditional DT methods that rely solely on vehicle's real-time state synchronization, however, might lead to unacceptable computing and communication consumption for construction of high-fidelity DT with redundant data. 
%safety based on a high-fidelity digital representation of the physical driving environment. 
%Moreover, the DT can accurately imitate mobile physical obstacles (such as vehicles, cyclists, and pedestrians) to infer trajectory adjustment decisions to ensure driving reliability.
%However, traditional push-based state-reporting DT methods ignore obstacle-state selection entirely to vehicles, which make the DT not to receive the reported information passively.
%These methods also face challenges in bandwidth-limited communication environments, as vehicles may upload states that have already been accurately simulated by the DTs while neglecting the DT's needs.
% This approach reduces the transmission overhead by eliminating states that are already accurately captured by the DTs' simulation, thereby significantly reducing communication overhead.
To address this issue, we first propose a query-driven DT architecture to enable the DT to actively request the desired environment data from vehicles based on its simulation result. 
Then, we formulate an optimization problem whose goal is to minimize autonomous driving position error while accounting for DT fidelity and communication constraints. 
% To support the architecture
% this paper proposes a query-driven DT architecture where we enable the DT to positively request desired environment data based on simulation results rather than passively receiving vehicles' information. To support the architecture
We also design a cross-time-step progressive query mechanism to further improve communication efficiency.
% The DT leverages proactive simulation to dispatch coarse and fine-grained queries across different future time steps. This way avoids excessive latency within a single time slot while preserving reliable responses.
% this mechanism dispatches a pair of coarse and fine queries into various time steps for proactive simulation.
% The optimization problem is further  
% through cross-attention over local perception features
%We then analytically characterize how the driving position error is upper-bounded by each obstacle’s impact on the vehicle’s driving plan and the staleness of such impact.
% Based on this, we reformulate the problem as minimizing the upper bound on driving position error by selecting the obstacles that have a strong effect on driving performance.
The simulation results show that our proposed method achieves a $24\%$ reduction in planning position error compared to traditional methods, while reducing communication overhead by $40\%$. 
 
\end{abstract}


\section{Introduction}\label{sec:intro}

Driven by their potential to improve safety and transportation efficiency, end-to-end automated driving (AD) systems that plan ego trajectories based on localized perception have attracted widespread attention \cite{DTAD,jiang2023vad,hu2023uniad}.
However, end-to-end AD is fundamentally constrained by blind spots and hardware capacity, leading to compromised safety \cite{xu2022v2xvit,NEURIPS2022_1f5c5cd0}. 
To fill this gap, digital twins (DTs) have emerged as a promising solution.
By constructing a digital replica of the driving environment without blind spots through collaborative perception, DT enables risk-free simulation and optimal trajectory planning based on abundant computation \cite{DTAD}.
% by constructing a digital driving environment based on collected vehicles' perceptions, which can reduce blind spots.
% And DT can perform risk-free imitation of physical entities and plan vehicles' trajectories in a digital driving environment constructed by collecting vehicles' perceptions . 
% Then, the DT can optimize the vehicle's trajectories through proactive simulation within the digital replica of the physical environment with sufficient computing resources.
Traditional DT methods, however, construct a reliable DT based on frequent state synchronization between the physical and digital worlds, which introduces unacceptable computational and communication overhead.
\begin{figure}[t]
\includegraphics[width=\linewidth]{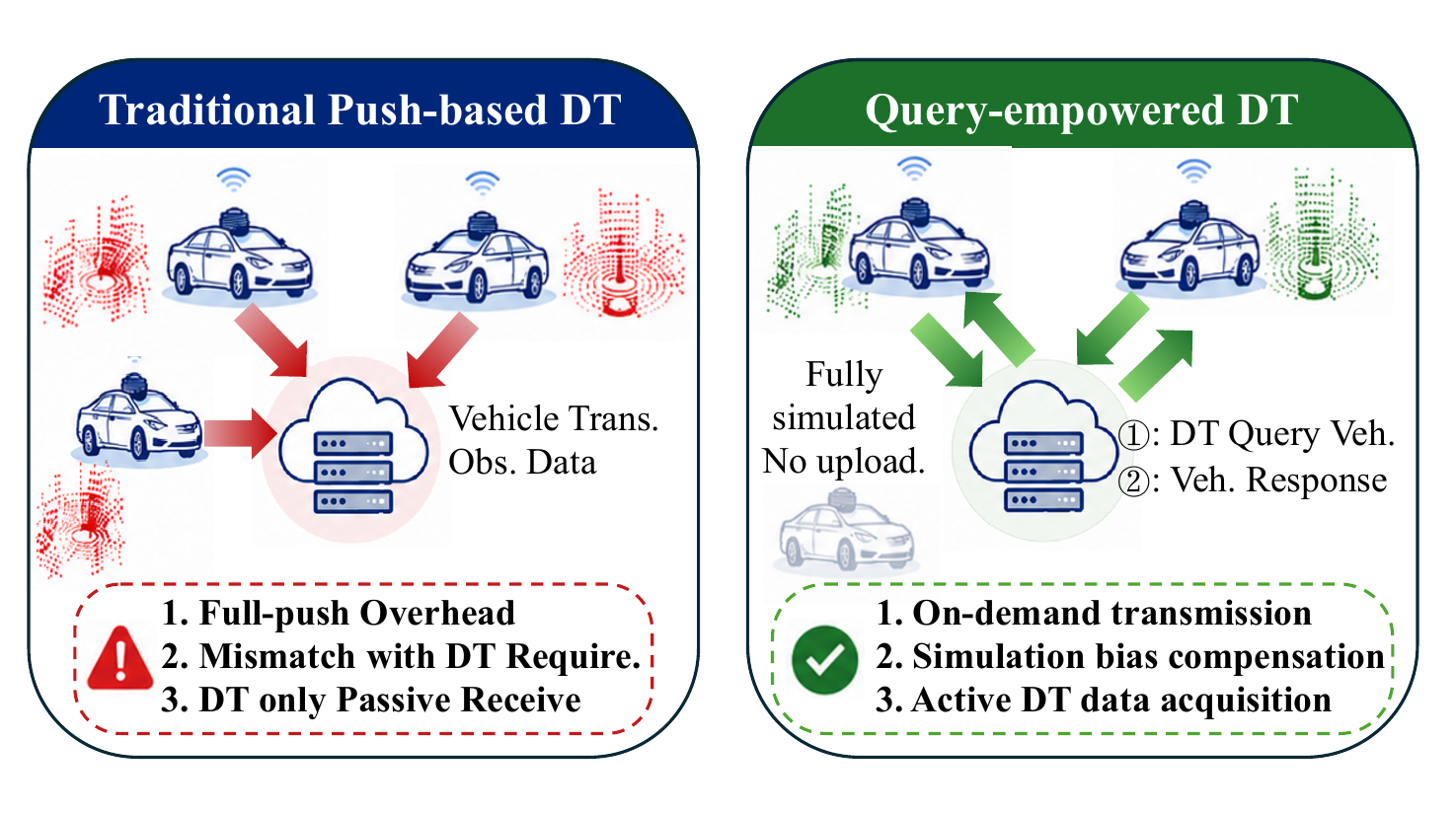}
\caption{Breakthroughs compared to existing work.}
\label{fig1}
\vspace{-0.5cm}
\end{figure}

To tackle this problem, a number of works \cite{zhou2026lowoverhead_dt, ToNAoI,TowardEfficientDeployment,hanzhiyu,HTongTMC,YoloDT,YanshaDGoalOriented} have been studied, which can be broadly categorized into two groups: state synchronization frequency optimization and state compression. 
For the first group, the authors in \cite{zhou2026lowoverhead_dt} proposed an age-of-information (AoI)-based method that only synchronizes the information whose staleness exceeds a predefined threshold.
Similarly, the authors in \cite{ToNAoI} proposed an approximation-optimization-based algorithm to balance state staleness and transmission overhead.
Unlike approaches that account for state staleness, the authors in \cite{TowardEfficientDeployment} proposed an event-triggered strategy that synchronizes only when state changes significantly.
To enable more flexible adoption across state-stale and event-triggered methods, the authors in \cite{hanzhiyu} proposed a multi-agent reinforcement learning (RL) method to optimize the state synchronization frequency.
However, these methods \cite{zhou2026lowoverhead_dt, ToNAoI,TowardEfficientDeployment,hanzhiyu,HTongTMC} relying solely on frequency optimization can lead to the transmission of task-irrelevant states, resulting in degraded transmission efficiency and unreliable planning.
For the second group, the authors in \cite{HTongTMC} proposed a semantic communication-based state compression method that can extract task-relevant features.
Similarly, the authors in \cite{YoloDT} employed a computer vision model to extract critical semantic information during state synchronization.
Beyond semantic-based extraction, the authors in \cite{YanshaDGoalOriented} proposed a goal-oriented feature selection method that can adapt to various sub-tasks.
However, due to their static design, these compression methods \cite{HTongTMC, YoloDT, YanshaDGoalOriented} lack the necessary flexibility to adapt to dynamic driving environments.
Furthermore, these works \cite{zhou2026lowoverhead_dt, hanzhiyu,ToNAoI,TowardEfficientDeployment,HTongTMC,YanshaDGoalOriented,YoloDT} can be summarized as \textit{push-based} DT paradigms, in which only physical entities can decide their perception report frequency and compression strategy.
Consequently, this paradigm fundamentally incurs excessive computational and communication overheads, as physical entities redundantly upload data streams that already overlap with the DT's simulation.
% This paradigm, in turn, may lead to fundamental drawbacks in substantial computational and communication overhead caused by physical entities' upload
% overlooked DT's simulation capability, which leads .
% In other words, the DT, which serves as the central brain for multi-vehicle AD, remains a passive consumer of whatever state the vehicles choose to upload, unable to steer the synchronization process toward the state it simulation lacks.

To bridge this gap, some works \cite{ADT,11317204,PullBasedQuery} have explored a new paradigm called \textit{active-based} DT, in which DT can determine which subset of the state to collect.
The authors in \cite{ADT} modeled environmental change as a partially observable Markov decision process and selected a subset of state information via DTs' simulation at each time step.
Following \cite{ADT}, the authors in \cite{11317204} enable the DT to request physical information when prediction discrepancies exceed a threshold.
Unlike the fix-error-triggered data collection in \cite{11317204}, the authors in \cite{PullBasedQuery} proposed an active request scheduling architecture that employs RL to select a subset of states while accounting for communication resources slot by slot.
By enabling the DT to actively determine which states to collect based on its current simulation belief, \textit{active-based} paradigms fully exploit the DT's simulation capability for task-oriented state synchronization.
However, these works \cite{ADT,11317204,PullBasedQuery} still face three key limitations in the AD scenario.
First, these works modeled the state synchronization problem as selecting a subset of structured or stable state variables (e.g., velocity, position).
% However, in an AD scenario, the states of driving environment and obstacles vary significantly across viewpoints and are often represented in latent feature spaces rather than in structured, stable formats.
However, in an AD scenario, environmental states are typically captured in latent spaces due to viewpoint dependence, rather than in structured formats.
Hence, without stable state variables, these methods are infeasible to directly extract the required information from the latent space.
Second, these methods still overlook the impact of drift caused by DTs' imperfect simulation. 
As imperfect DT simulations inherently introduce prediction errors, subsequent state synchronization decisions may further amplify the perceptual mismatch between the DT's desired and the physical environment. 
Finally, their slot-wise and independent state-collection strategies fail to meet the latency requirements in AD systems. 
% As a result, the resulting end-to-end query-response pipeline may incur substantial latency and .
% Although these active DT methods \cite{ADT,PullBasedQuery,11317204,SemaAwareDigitalTwin} build on the same principle of closing the loop between perception and state collection,但他们面临三个关键问题。
% 首先但是采集的一直都是明文信息，不涉及复杂的系统，或者换句话说，只是一种状态掩码策略，本质上都是利用DT对数据中的部分信息进行掩藏实现主动的采集，而在自动驾驶驱动的数字孪生系统中，车辆，目标，障碍物等的状态从不同视角看存在巨大的差异，并且当前自动驾驶系统并没有先对车辆状态进行表格化的处理（标注速度，位置等），而是都是基于隐空间的表征，无法通过明文的要求描述出来。
% 此外，这些方法没有充分考虑数字孪生系统中存在的漂移问题，因为数字孪生的推演和预测本身会引入误差，而在此基础上进行的数据采集就会出现漂移。
% 最后，这些每一个时隙进行单独信息采集的策略存在巨大的时延，比如难以满足自动驾驶数字孪生对时延的需求。
% However, 
% , enabling these systems to autonomously improve situational awareness, refine their internal models through active exploration, and proactively manage the evolution of the environment.
% However, the 作者对待同步的信息的价值直接进行了估计，并采集价值高的目标，虽然其也利用了DT的仿真能力，但是采集的一直都是明文信息，不涉及复杂的系统。
% 在自动驾驶驱动的数字孪生系统中，车里那个，目标，障碍物等的状态都是基于隐空间的表征，无法通过明文的要求描述出来，并且，他们具有强烈的观测者单一视角的影响，不能作为通用的数据表征。

To enable state collection from unstable latent features, tolerate DTs' drift simulation, and ensure per slot communication latency, we propose a novel query-driven DT architecture.
In particular, the proposed architecture enables DT to request the desired environment data based on its simulation, which can adapt to unstable latent features through a cross-attention mechanism as shown in Fig. \ref{fig1}.
% Specifically, DT first simulates environmental changes and transmits a query tensor to request the state of obstacles that have the greatest impact on the future trajectory planning error.
To further enhance the query-driven DT's tolerance to simulation drift and reduce communication latency per slot, we propose a progressive, cross-timestep query-response mechanism.
Compared to the traditional active DT state-collection architecture that can only collect states from structured, stable formats, the proposed method can collect information from latent feature spaces while accounting for DT drift and latency requirements.
Our contributions are summarized as follows.
\begin{itemize}
    \item 
    We propose a novel query-driven DT architecture that enables DT to actively request the desired information based on its simulation to minimize the trajectory planning error.
    Compared to traditional DT paradigms, the proposed query-driven architecture can effectively enhance on-demand data collection and avoid redundant data transmissions.
    % Meanwhile, by leveraging its simulation capabilities and targeted information acquisition, the DT can efficiently reconstruct a high-fidelity replica of a digital driving environment.
    \item We formulate an optimization problem to minimize the DT's planning position error, accounting for DT fidelity and communication resources.
    To solve this problem, we first analyze each entity's effect on the DT's planning position error through derivation. 
    Based on this, we reformulate the original problem as an entity selection and communication resource allocation problem, which is then solved using a query-response mechanism. 
    \item To further reduce computing overhead during entity selection, we first design an entity's effect estimation model.
    Then, a progressive query method that contains coarse and fine-grained queries is proposed to enhance stability under DT's simulation drift.
    Finally, we propose a cross-timestep query--response mechanism and a joint communication resource scheduling method to fully exploit communication resources across frames, thereby avoiding excessive latency. Our method also reduces trajectory-planning position error by up to $24\%$ compared with traditional algorithms.
\end{itemize}
%Extensive experiments show that our method can effectively reduce communication overhead while reducing trajectory-planning position error by up to $24\%$ compared with traditional algorithms.
% \subsection{Background}
% 我们不写related work，写成Background，只用一小片地方说明DT的作用，以及自动驾驶过程中，为什么感知是连续的状态
% \textit{Notations:} Unless otherwise indicated, matrices are represented by bold capital letters (i.e. $\bf{A}$), vectors are denoted by bold lowercase letters (i.e. $\boldsymbol{v}$), and scalars are denoted by plain font (i.e. $d$). 
% The term $||\boldsymbol{w}||$ donates the L2-norm and $|\boldsymbol{w}|$ represents the L1-norm of $\boldsymbol{w}$.
\mysection{Related Work}\label{sec:related}

\mysubsection{Scene Representation and Collaborative Perception for Autonomous Driving}
Recently, several works have studied end-to-end AD systems that generate trajectory plans by analyzing locally collected, complex, and highly dynamic images and radar data \cite{survey_worldmodel2025, hu2023uniad,jiang2023vad,li2022bevformer}.
The authors in \cite{hu2023uniad} proposed a query-based architecture that passes a shared set of agent and map queries through successive tracking, mapping, motion, and occupancy transformers into a final planner.
% During tracking, mapping, motion, and occupancy, the environment's state (e.g., dynamic obstacles and a static map) is represented as latent features rather than as explicit variables.
The authors in \cite{jiang2023vad} introduced a fully vectorized representation map with vectorized map and agent tokens, encoding the surroundings as compact instance-level latents for planning.
The authors in \cite{li2022bevformer} proposed a bird's-eye-view (BEV) encoder that lifts multi-camera features into a plane via spatial cross-attention and aggregates history through temporal self-attention over recurrent BEV queries.
These works represent the driving environment (e.g., dynamic obstacles and a static map) as continuous, viewpoint-dependent latent features instead of structured variables such as position or velocity, which makes the perceptual state difficult to specify and synchronize explicitly.

To extend perception beyond a single viewpoint, a number of works have explored collaborative perception that fuses such latent features across vehicles~\cite{tian2026uncertainty}.
The authors in \cite{wang2020v2vnet} proposed a spatially-aware graph neural network that warps and message-passes intermediate feature maps among neighboring vehicles to compensate for occlusion and pose misalignment.
The authors in \cite{xu2022v2xvit} introduced a unified vision transformer with heterogeneous multi-agent self-attention and multi-scale window attention to fuse features across vehicles and infrastructure of differing viewpoints.
The authors in \cite{hu2022where2comm} proposed a spatial-confidence-map mechanism that gates communication so that only perception-critical spatial regions are packed and exchanged.
The authors in \cite{Cooptrack} introduced a fully instance-level end-to-end framework with learnable instance association that transmits sparse instance-level features.
Although these works fuse latent perceptual features across viewpoints under communication resources constraints, every state information exchange is determined by the vehicles, without anticipating which states will be most informative for downstream reasoning.
Further, these single-vehicle and peer-to-peer schemes remain constrained by blind spots and provide no central entity that maintains a globally consistent, long-horizon view of the scene.

\mysubsection{Digital Twins for Automated Driving}
Recently, DTs have been widely introduced as digital replicas that can provide a globally consistent, long-horizon view of the scene, integrating multi-vehicle perception for risk-free simulation and proactive control \cite{bariah2023dt_empowered}.
The authors in \cite{wang2024smdt} proposed a smart-mobility DT that fuses LiDAR and camera streams from roadside units over multi-band vehicle-to-everything (V2X) and mobile-edge computing into a cloud replica, on which real-time collision prediction and route planning are executed for connected vehicles.
The authors in \cite{tang2024tvt} formulated a joint optimization of vehicular sensing-and-upload selection and vehicle-DT placement across a cloud--edge hierarchy, and solved it with RL to balance twin accuracy against real-time deployment cost.
These works demonstrate that a vehicular DT is necessary to overcome single-vehicle blind spots, yet they treat the twin as a passive consumer of whatever state the vehicles choose to upload.
Overall, the DT provides the central reasoning entity that single-vehicle perception lacks, but maintaining a faithful twin requires continuously reconstructing the high-dimensional, viewpoint-dependent perceptual state under a tight communication budget.

\mysubsection{Efficient and Active Digital Twin Synchronization}
To reconstruct the DT under limited communication resources, existing works have investigated efficient state synchronization between the physical and digital worlds.
The authors in \cite{zhou2026lowoverhead_dt} proposed an age-of-information-based scheme that synchronizes a state only when its staleness exceeds a predefined threshold.
The authors in \cite{TowardEfficientDeployment} introduced an event-triggered strategy that uploads only significant state changes.
The authors in \cite{HTongTMC} proposed a semantic-communication-based compression that extracts task-relevant features before synchronization.
The authors in \cite{YanshaDGoalOriented} introduced a goal-oriented feature and temporal selection scheme that filters task-irrelevant information.
However, these works reduce synchronization costs on the transmitter side by adopting a push-based paradigm in which the physical entities decide what to upload, overlooking the DT's simulation capability.

From a different perspective, recent works have explored active synchronization, which allows the twin to decide which states to collect based on its current simulation belief~\cite{GoalOrientedAccess}.
The authors in \cite{ADT} modeled environmental change as a partially observable Markov decision process and selected states via active inference to minimize simulation uncertainty.
The authors in \cite{11317204} enabled the twin to request physical information only when the prediction discrepancy exceeds a threshold.
The authors in \cite{PullBasedQuery} proposed a pull-based request-scheduling architecture that employs RL to select a subset of states while accounting for communication resources.
Although these works exploit the twins' belief to drive state collection, they assume the state is structured and stable, overlook the drift introduced by imperfect simulation, and act slot-by-slot without temporal coordination.
Further, these works do not consider the latent-state driving twin under strict latency constraints.

\section{System Model}\label{sec:system}

We consider an AD scenario comprising $N$ vehicles and a central server $S$, which deployed a DT service at the base station that can connect to all vehicles via wireless links.
The set of vehicles is indexed as $\mathcal{N} = \{1,2,\cdots,N\}$.
The DT-enhanced AD system operates through the following three-step workflow, as shown in Fig. \ref{fig:system_model}:
\begin{figure*}[thp]
\centering
\includegraphics[width=0.85\linewidth]{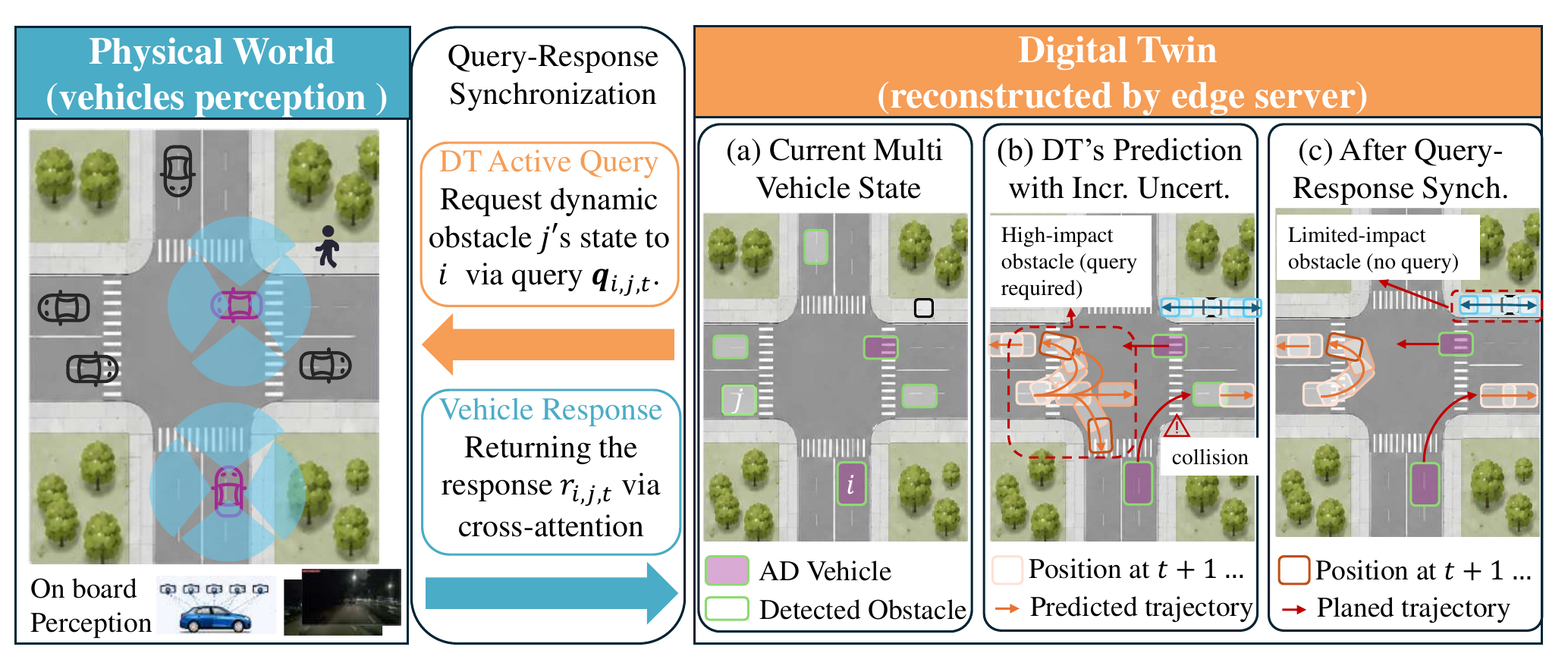}
\caption{Illustration of the proposed query-empowered DT in an AD scenario.}
\label{fig:system_model}
\end{figure*}
\begin{enumerate}
  \item \textbf{Vehicle-Side AD Pipeline:}
    Each vehicle equipped with an onboard autonomous-driving pipeline executes low-level trajectory planning based on single-vehicle observations, which is limited by a single viewpoint.
  \item \textbf{Query-Response-based DT State Synchronization}:
    DT generates and transmits queries to vehicles to request the desired environment state, and then updates its digital environment replicas based on the vehicles' responses.
  \item \textbf{DT Simulation and Trajectory Planning}:
    DT further simulates environmental changes using a simulation process, then executes high-level trajectory planning across viewpoints.
\end{enumerate}

\subsection{Vehicle-Side Autonomous Driving Pipeline}\label{sec:vehicle}

We employ an end-to-end AD model on each vehicle to execute low-level trajectory planning, including perception, prediction, and planning through a unified architecture \cite{hu2023uniad}.

\textbf{Single-Vehicle Perception:} Let $\bm{x}_{i,t} \in \R^{N_{\mathrm{cam}} \times H_{\mathrm{img}} \times W_{\mathrm{img}} \times 3}$ denote the multi-camera images collected by vehicle~$i$ at time step $t$, where $N_{\mathrm{cam}}$ is the number of cameras, $H_{\mathrm{img}}$ and $W_{\mathrm{img}}$ are the spatial dimensions of images.
Then, each vehicle reconstructs a bird's-eye view (BEV) feature map by fusing $N_{\mathrm{cam}}$ images into a comprehensive representation of the surrounding environment, which is given by
\begin{equation}\label{eq:bev_encoder}
  \bm{B}_{i,t} = f_{\mathrm{bev}}(\bm{x}_{i,t}).
\end{equation}
$\bm{B}_{i,t} \;\in\; \R^{H_b \times W_b \times D}$, where $D$ is the feature dimension, $H_b$ and $W_b$ are the spatial dimensions of BEV feature map.

Then, each vehicle $i$ extracts the detectable dynamic obstacles (e.g., vehicles, cyclists, and pedestrians) and static map elements (e.g., road boundaries, lanes) from $\bm{B}_{i,t}$. 
The dynamic obstacles' representations are formulated as
\begin{equation}\label{eq:cross_attn_objects}
  \bm{z}_{i,t}^{\mathrm{obst}} = \bm{q}_{\mathrm{obst}} + \mathrm{softmax}\!\Bigl(\frac{\bm{q}_{\mathrm{obst}}\,\bm{B}_{i,t}^{\top}}{\sqrt{D}}\Bigr)\,\bm{B}_{i,t},
\end{equation}
where $\bm{q}_{\mathrm{obst}}\in \R^{N_{\mathrm{obst}} \times D}$ is the $N_\mathrm{obst}$ learnable obstacle detection queries.
$\bm{z}_{i,t}^{\mathrm{obst}}=[\bm{z}_{i,1,t}^{\mathrm{obst}},\cdots,\bm{z}_{i,N_\mathrm{obst},t}^{\mathrm{obst}}]$ contains each dynamic obstacle's representations.
Similarly, static maps' representations are given by
\begin{equation}\label{eq:cross_attn_map}
  \bm{z}_{i,t}^{\mathrm{map}} = \bm{q}_{\mathrm{map}} + \text{Softmax}\!\Bigl(\frac{\bm{q}_{\mathrm{map}}\,\bm{B}_{i,t}^{\top}}{\sqrt{D}}\Bigr)\,\bm{B}_{i,t},
\end{equation}
where $\bm{q}_{\mathrm{map}}\in\R^{N_{\mathrm{map}} \times D}$ is the $N_\mathrm{map}$ learnable map queries and $\bm{z}_{i,t}^{\mathrm{map}}=[\bm{z}_{i,1,t}^{\mathrm{map}},\dots,\bm{z}_{i,N_\mathrm{map},t}^{\mathrm{map}}]$.

\textbf{Prediction and Planning:} Then, each vehicle $i$ predicts the trajectories of detected obstacles, which are given by
\begin{equation}\label{eq:prediction}
    \begin{aligned}
\{  {\bm{Y}}_{i,t},\, \bm{\pi}_{i,t},&\,\bm{z}_{i,t}^{\mathrm{ctx}},\,\bm{z}_{i,t}^{\mathrm{moti}} \}
  = f_{\mathrm{moti}}(\bm{z}_{i,t}^{\mathrm{obst}},\, \bm{z}_{i,t}^{\mathrm{map}}),
    \end{aligned}
\end{equation}
where ${\bm{Y}}_{i,t} \in \R^{N_{\mathrm{obst}} \times T_f \times P \times 2}$ contains predicted 2D positions over a horizon of $T_f$ steps of $P$ trajectory modes, $\bm{\pi}_{i,t} \in \R^{N_{\mathrm{obst}} \times P}$ stacks the prediction probability of $P$  trajectories.
$\bm{z}_{i,t}^{\mathrm{ctx}}$ is a motion context embedding, $\bm{z}_{i,t}^{\mathrm{moti}}=[\bm{z}_{i,1,t}^{\mathrm{moti}},\cdots,\bm{z}_{i,N_{\mathrm{obst}},t}^{\mathrm{moti}}]\in\R^{N_{\mathrm{obst}}\times P\times D}$ stacks each obstacle's motion representations. Given the predicted trajectory of all detected obstacles, each vehicle $i$ then plans its own trajectory as
\begin{equation}\label{eq:planning}
  \bm{\tau}_{i,t} = f_{\mathrm{plan}}(\bm{z}_{i,t}^{\mathrm{ctx}}, \, {\bm{Y}}_{i,t}, \, \bm{g}_{i,t}),
\end{equation}
where $\bm{g}_{i,t}$ is a navigation command vector.

\subsection{Query-Response-based DT State Synchronization}\label{sec:dt_state}

We assume the DT holds a static map representation $\hat{\bm{z}}^{\mathrm{map}}$, which directly retrieves static map elements from the offline HD-map~\cite{liao2023maptr} and remains unchanged throughout operation.
Hence, to maintain a high-fidelity replica of the physical environment, DT only needs to synchronize the dynamic obstacle states.
However, since each vehicle only maintains local perception $\bm{z}_{i,t}^{\mathrm{obst}}$, which is inherently viewpoint-dependent, it must be further fused across vehicles to obtain a comprehensive representation.  
Next, we propose a query-response-based DTs state-synchronization mechanism for initialization and evaluation with a multi-vehicle perception-fusion method.

At the first frame of DT construction, DT performs a state initialization. 
Let $\mathcal{U}_{j,t}$ denote the vehicle set that can detect obstacle $j$ at time slot $t$, and $M_t$ be the number of tracked obstacles.
The obstacles set $\mathcal{M}_t=\{1,\cdots,M_t\}$.
Specifically, the DT initializes per-vehicle memories by a query-response mechanism as
\begin{equation}\label{eq:init_h}
  \bm{r}_{i,j,0} = \mathrm{CrossAttention}\bigl(\bm{q}_{\mathrm{init}},\;  \{\bm{z}^{\mathrm{loc}}_{i,j,t}\}_{i\in\mathcal{U}_{j,0}},\; \{\bm{z}^{\mathrm{loc}}_{i,j,t}\}_{i\in\mathcal{U}_{j,0}}\bigr),
\end{equation}
where $\mathrm{CrossAttention}(\cdot,\cdot,\cdot)$ is the cross attention operation \cite{vaswani2017attention}.
$\bm{q}_{\mathrm{init}} \in \mathbb{R}^{D}$ is a learnable initialization query shared across all vehicles and obstacles.
$\bm{z}^{\mathrm{loc}}_{i,j,t} = [\,\bm{z}^{\mathrm{obst}}_{i,j,t},\; \bm{z}^{\mathrm{moti}}_{i,j,t}\,]$ is vehicle~$i$'s local representation stack for obstacle~$j$.
The DT then fuses these initial responses into a global belief given by
\begin{equation}
\label{eq:latent_fusion}
  \hat{\bm{z}}_{j,0}^{\text{agg}} = \mathrm{CrossAttention}\bigl( \bm{q}_{\mathrm{fuse}},\;  \{\bm{r}_{i,j,0}\}_{i\in\mathcal{U}_{j,0}},\; \{\bm{r}_{i,j,0}\}_{i\in\mathcal{U}_{j,0}}\bigr),
\end{equation}
where $\bm{q}_{\mathrm{fuse}} \in \mathbb{R}^{D}$ is the feature for fusing and \(\hat{\bm{z}}_t^{\mathrm{agg}}=[\hat{\bm{z}}_{1,t}^{\mathrm{agg}},\cdots,\hat{\bm{z}}_{M_t,t}^{\mathrm{agg}}]\).

In the subsequent frame $t$ after initialization, DT must actively select a subset of obstacles to synchronize their states while the remaining obstacles are left unsynchronized under limited communication resources.
This decision is formalized through indicator $u_{i,j,t} \in \{0,1\}$, where $u_{i,j,t}=1$ implying that DT query vehicle $i$ for obstacles $j$' state at frame $t$, and $u_{i,j,t}=0$, otherwise.
Then we can define per-obstacle's staleness to measure its effectiveness as
\begin{equation}\label{eq:staleness}
  \Delta_{j,t} = t - \max\Bigl\{\, \tau \;\Big|\; \sum_{i \in \mathcal{U}_{j,\tau}} u_{i,j,\tau} > 0 \,\Bigr\}.
\end{equation}

When DT select vehicle $i$ for state of obstacle $j$ (i.e. $u_{i,j,t}=1$), the DT generate and transmit query $\bm{q}_{i,j,t}\in\R^{d\times D}$ to vehicle~$i$ with $d \ll D$. Upon receiving the query, vehicle~$i$ computes the response $\bm{r}_{i,j,t}$ via the cross-attention as
% \begin{equation}\label{eq:q2_readout}
  % \bm{r}_{i,j,t} \;=\; \text{softmax}\left(\frac{(\bm{W}_{k}\bm{z}^{\mathrm{loc}}_{i,j,t})^{\top}\bm{q}_{i,j,t}}{\sqrt{d_q}}\right)\bigl(\bm{W}_{v}\bm{z}^{\mathrm{loc}}_{i,j,t}\bigr),
% \end{equation}
\begin{equation}\label{eq:q2_readout}
  \bm{r}_{i,j,t} \;=\; \mathrm{CrossAttention}\bigl( \bm{q}_{i,j,t},\;  \bm{z}^{\mathrm{loc}}_{i,j,t},\; \bm{z}^{\mathrm{loc}}_{i,j,t}\bigr).
\end{equation}
% where $\bm{W}_{k}\!\in\!\R^{d_q\times D}$, $\bm{W}_{v}\!\in\!\R^{d_v\times D}$ are learnable shared projections matrix with $d_v \ll D$.

The DT then fuses all available per-vehicle responses into the global belief via the same fusion mechanism as in initialization:
\begin{equation}\label{eq:per_frame_fusion}
  \hat{\bm{z}}_{j,t}^{\mathrm{agg}} = \mathrm{CrossAttention}\bigl( \bm{q}_{\mathrm{fuse}},\;  \{\bm{r}_{i,j,t}\}_{i\in\mathcal{U}_{j,t}},\; \{\bm{r}_{i,j,t}\}_{i\in\mathcal{U}_{j,t}}\bigr).
\end{equation}

Given the global belief $\{\hat{\bm{z}}_{1,t}^{\mathrm{agg}},\cdots,\hat{\bm{z}}_{M_t,t}^{\mathrm{agg}}\}$, the motion model predicts multimodal trajectories as
\begin{equation}\label{eq:dt_traj_pred}
  \hat{\bm{Y}}_t,\; \hat{\bm{\pi}}_t,\; \hat{\bm{z}}_t^{\mathrm{ctx}},\; \hat{\bm{z}}_t^{\mathrm{moti}}
  =
  f_{\mathrm{moti}}(\hat{\bm{z}}_t^{\mathrm{agg}},\; \hat{\bm{z}}^{\mathrm{map}}),
\end{equation}
where
\(\hat{\bm{Y}}_t \in \mathbb{R}^{M_t\times P \times T_f \times 2}\)
contains \(P\) candidate trajectories for each obstacle over a prediction horizon of \(T_f\) frames.
The matrix
\(\hat{\bm{\pi}}_t=[\hat{\bm{\pi}}_{1,t},\cdots,\hat{\bm{\pi}}_{M_t,t}]
\in [0,1]^{M_t \times P}\)
denotes the corresponding mode probabilities, where
\(\sum_{p=1}^{P} [\hat{\bm{\pi}}_{j,t}]_{p}=1\).
Based on this motion context, the DT plans a trajectory for each vehicle as
\begin{equation}\label{eq:dt_planning}
  \hat{\bm{\tau}}_{i,t} = f_{\mathrm{plan}}(\hat{\bm{z}}_t^{\mathrm{ctx}},\; \hat{\bm{Y}}_{t},\; \bm{g}_{i,t}).
\end{equation}

While obstacle $j$ has not been queried for several frames, DT cannot treat its aggregated state as a deterministic latent vector.
Instead, the state of obstacle $j$ should be regarded as a belief distribution induced by the multimodal trajectory hypotheses predicted in (\ref{eq:dt_traj_pred}).
Specifically, for each motion mode \(p\in\{1,\ldots,P\}\), the DT obtains a mode-conditioned belief branch as
\begin{equation}\label{eq:mode_conditioned_belief}
  \hat{\bm{z}}_{j,t+1}^{\mathrm{agg},(p)}
  =
  f_{\mathrm{simulate}}\bigl(
  \hat{\bm{z}}_{j,t}^{\mathrm{agg}},
  \hat{\bm{Y}}_{j,t}^{(p)},
  \hat{\bm{z}}^{\mathrm{map}}
  \bigr).
\end{equation}
Since the true motion mode of an unqueried obstacle is unknown, the DT can only propagate the state to the next frame, which can be interpreted as a random variable with a categorical distribution over modes.
Equivalently, the DT sample a trajectory $s_{j,t} \sim \mathrm{Categorical}(\hat{\bm{\pi}}_{j,t})$, and obtain one realization of the next aggregated state as $\hat{\bm{z}}_{j,t+1}^{\mathrm{agg},(s_{j,t})}$.
% The full propagated belief at the next frame can therefore be represented as a categorical mixture of all mode-conditioned belief branches as
% \begin{equation}\label{eq:belief_mixture}
%   \mathcal{B}_{j,t+1}^{\mathrm{agg}}
%   =
%   \left\{
%   \left(
%   \hat{\bm{z}}_{j,t+1}^{\mathrm{agg},(p)},
%   [\hat{\bm{\pi}}_{j,t}]_p
%   \right)
%   \right\}_{p=1}^{P}.
% \end{equation}
Over $\Delta_{j,t}$ unqueried frames, this belief propagation is applied recursively.
Then, the DT belief of obstacle $j$ is induced by a sequence of sampled motion modes as
\begin{equation}
  \bm{s}_{j,t-\Delta_{j,t}:t-1}
  =
  \left(
  s_{j,t-\Delta_{j,t}},
  \ldots,
  s_{j,t-1}
  \right).
\end{equation}
The probability of $\bm{s}_{j,t-\Delta_{j,t}:t-1}$ is given by
\begin{equation}\label{eq:probability}
  \Pr\!\left(
  \bm{s}_{j,t-\Delta_{j,t}:t-1}
  \right)
  =
  \prod_{\tau=t-\Delta_{j,t}}^{t-1}
  [\hat{\bm{\pi}}_{j,\tau}]_{s_{j,\tau}}.
\end{equation}
The corresponding belief branch at frame \(t\) is written compactly as
\begin{equation}
    \begin{aligned}
      &\hat{\bm{z}}_{j,t}^{\mathrm{agg},\bm{s}_{j,t-\Delta_{j,t}:t-1}} 
      \\ &=
      f_{\mathrm{simulate}}^{\Delta_{j,t}}
      \left(
      \hat{\bm{z}}_{j,t-\Delta_{j,t}}^{\mathrm{agg}},
      \left\{
      \hat{\bm{Y}}_{j,\tau}^{(s_{j,\tau})}
      \right\}_{\tau=t-\Delta_{j,t}}^{t-1},
      \hat{\bm{z}}^{\mathrm{map}}
      \right).
    \end{aligned}
\end{equation}
Thus, the DT belief over an unqueried obstacle evolves from a synchronized latent state into a distribution over possible latent states.
This distribution is recursively induced by the motion-mode sequence, and its uncertainty accumulates as the synchronization gap increases.
% \begin{equation}\label{eq:tree_step}
%   \hat{\bm{z}}_{j,t}^{\mathrm{agg},(p_{t-\Delta_{j,t}},\cdots,p_{t})} = f_{\mathrm{simulate}}\bigl(\hat{\bm{z}}_{j,t-\Delta_{j,t}}^{\mathrm{agg},(p_{t-\Delta_{j,t}},\cdots,p_{t-1})},\; \hat{\bm{Y}}_{j,t}^{(p_{t})},\; \hat{\bm{z}}^{\mathrm{map}}\bigr).
% \end{equation}

\subsection{Wireless Channel Model}
To account for query-response transmission overhead, we adopt a standard urban large-scale channel model.
Let the distance between vehicle~$i$ at position $\bm{v}_{i,t} \in \R^3$ and DT server $S$ at position $\bm{p}_{\mathrm{server}} \in \R^3$ is $d^{\mathrm{2D}}_{i,t} = \|[\bm{v}_{i,t}]_{1:2} - [\bm{p}_{\mathrm{server}}]_{1:2}\|_2$ and $d^{\mathrm{3D}}_{i,t} = \|\bm{v}_{i,t} - \bm{p}_{\mathrm{server}}\|_2$.
The line-of-sight (LOS) probability is distance-dependent in the urban scenario 2D plane as $\prob_{\mathrm{LOS}}(d^{\mathrm{2D}}_{i,t}) = \min(d_1/{d^{\mathrm{2D}}_{i,t}}, 1)(1 - e^{-{d^{\mathrm{2D}}_{i,t}}/d_2}) + e^{-{d^{\mathrm{2D}}_{i,t}}/d_2}$, where $d_1$ and $d_2$ are scenario-dependent parameters.
% \begin{equation}
% \label{eq:pathloss}
%   \mathrm{PL}({d^{\mathrm{3D}}_{i,t}}) \!\!=\!\!\! \begin{cases}    \alpha_{\mathrm{L}}\!\!+\!\!\beta_{\mathrm{L}}\log_{10}({d^{\mathrm{3D}}_{i,t}}) + \gamma\log_{10}(f_c), {\text{w.p. } \prob_{\mathrm{LOS}}(d^{\mathrm{2D}}_{i,t})}, \\  \alpha_{\mathrm{N}}\!\!+\!\!\beta_{\mathrm{N}}\log_{10}({d^{\mathrm{3D}}_{i,t}}) + \gamma\log_{10}(f_c), \text{otherwise},
%   \end{cases}
% \end{equation}
% where $d_1$ and $d_2$ are scenario-dependent parameters.
Accordingly, the large-scale path loss is
\begin{equation}
\label{eq:pathloss}
\mathrm{PL}_{i,t}
=
\begin{cases}
\alpha_{\mathrm{L}}
+
\beta_{\mathrm{L}}\log_{10}(d^{\mathrm{3D}}_{i,t})
+
\gamma\log_{10}(f_c),
& \text{w.p. } \prob_{\mathrm{LOS}}(d^{\mathrm{2D}}_{i,t}),\\
\alpha_{\mathrm{N}}
+
\beta_{\mathrm{N}}\log_{10}(d^{\mathrm{3D}}_{i,t})
+
\gamma\log_{10}(f_c),
& \text{otherwise},
\end{cases}
\end{equation}
where $f_c$ is the carrier frequency and {$\alpha_{\mathrm{L}}, \beta_{\mathrm{L}}, \alpha_{\mathrm{N}}, \beta_{\mathrm{N}}$}, $\gamma$ are model-specific coefficients.

Let
$\boldsymbol{A}_{t}=[A_{1,t},\ldots,A_{N,t}]$
denote the RB allocation vector at frame $t$, where
$\sum_{i=1}^{N}A_{i,t}\le R$ and $A_{i,t}\in\{0,\ldots,R\}$.
Given $A_{i,t}$ RBs, the transmission rate of vehicle $i$ is
\begin{equation}
\label{eq:rate}
r_{i,t}= A_{i,t}B_w
\log_2
\left(
1+
\frac{P_{\mathrm{tx}}}
{B_w N_0 L_{i,t}}
\right),
\end{equation}
where $B_w$ is the bandwidth of each RB, $P_{\mathrm{tx}}$ is the transmit power,
$N_0$ is the noise power spectral density, and
$L_{i,t}=10^{\mathrm{PL}_{i,t}/10}$ is the linear-scale path loss.
Therefore, the delay for transmitting the query-responses assigned to vehicle $i$ is
\begin{equation}
\label{eq:delay}
\mathcal{T}(\boldsymbol{A}_{t},\bm{v}_{i,t},C)
=
\frac{\sum_{j\in\mathcal{O}_{i,t}}u_{i,j,t}C}{r_{i,t}},
\end{equation}
where $C$ is the payload size of each response and
\textcolor{black}{$\mathcal{O}_{i,t}=\{j\in\mathcal{M}_t\mid u_{i,j,t}=1\}$ denotes the set of obstacles assigned to vehicle $i$.
% 是不是没定义\mathcal{O}_{i,t}？\sum_{j\in\mathcal{O}_{i,t}}C
}

% The data transmission delay is 
% \begin{equation}\label{eq:rate}
%   \mathcal{T}({\boldsymbol{A}_{t}},\bm{v}_{i,t},C) = \sum_{j=1}^{|\mathcal{O}_{i,t}|}\frac{u_{i,j,t} \cdot C}{{{A}_{i,t}}\, B_w \log_2\!\left(1 + \frac{P_{\mathrm{tx}}}{{A}_{i,t}B_w N_0\mathrm{PL}({d^{\mathrm{3D}}_{i,t}})}\right)} 
% \end{equation}
% where $\boldsymbol{A}_{t}=\left[{A}_{1,t},\dots,{A}_{N,t}\right]$ is the resources blocks (RBs) allocation vector of the DT server at frame $t$ and $\sum_{i=1}^{N} {A}_{i,t}\leq R$ with ${A}_{i,t}\in\left[0,\cdots,R\right]$.
% $R$ is the total number of RBs, and $C$ is the transmission data size vehicle $i$ needs to transmit between vehicles and DT server at time frame $t$.
% $\mathcal{O}_{i,t}\subseteq \mathcal{M}_t$ is the index set of obstacle that vehicles $i$ can observer.
% $B_w$ is the unit channel bandwidth of each RB, $P_{\mathrm{tx}}$ is the transmit power, and $N_0$ is the noise power spectral density.
% $u_{i,j,t}\in\left{0,1\right}$ indicate the state synchronization vector with $u_{i,j,t}=1$ implying that DT require vehicle $i$ for obstacle $j$'s state at frame $t$ (i.e., ${A}_{i,t}>0$), and $u_{i,j,t}=0$, otherwise.

\subsection{Problem Formulation}\label{sec:problem}
We formulate an optimization problem to minimize the trajectory planning gap between the DT constructed via the proposed query-response state synchronization method and an idealized DT, under communication resource constraints, which is given by
\begin{align}
  \min_{\bm{u}_t,\, \{\bm{q}_{i,j,t}\},\, \boldsymbol{A}_{t}} & \quad
  \sum\limits_{{t} \in \mathcal{T}} \sum\limits_{i=1}^{N}  
  \E[\|\hat{\bm{\tau}}_{i,t} - \bm{\tau}_{i,t}^{\ast}\|^2], \label{eq:objective} \\
  \text{s.t.} \quad & \mathcal{T}({\boldsymbol{A}_{t}},\bm{v}_{i,t},C) \leq \mathcal{T}_{\max}, \forall\, i \in \mathcal{N}, \forall t\in \mathcal{T}, \tag{\theequation a}\label{eq:c1}\\
  & \E[\|\hat{\bm{z}}_{j,t}^{\ast} - \hat{\bm{z}}_{j,t}^{\mathrm{agg}}\|^2] \leq \sigma_{\max}, \forall\, j \in \mathcal{M}_t,
  \tag{\theequation b}\label{eq:c3} \\
  & \sum\limits_{i=1}^{N} {A}_{i,t}\leq R, \forall t\in \mathcal{T}, \tag{\theequation c}\label{eq:c4}
\end{align}
where $\bm{\tau}_{i,t}^{\ast}$ denotes the trajectory produced based on an idealized DT that possesses infinite fidelity and instantaneous synchronization. 
Similarly, we define the idealized state $\hat{\bm{z}}_{j,t}^{\ast}$ as the aggregated obstacle $j$'s representation with idealized DT.
$\mathcal{T}_{\max}$ and $\sigma_{\max}$ are the maximum per-frame latency and allowable state reconstruction error.
Constraint~\eqref{eq:c1} bounds the per-vehicle communication delay. 
Constraint~\eqref{eq:c3} constrains the per-obstacle reconstruction gap.
Constraint~\eqref{eq:c4} ensures the available resource blocks with maximum $R$ at each time slot.
The problem in~\eqref{eq:objective} is challenging to solve due to two key issues.
First, the relationship between the trajectory planning gap and each obstacle is difficult to explicitly characterize, making direct optimization intractable.
Second, the decision space is highly coupled and combinatorial, jointly involving active query design $\bm{q}_{i,j,t}$, and resource allocation $\boldsymbol{A}_t$.

\section{Problem Analysis}

To simplify problem (\ref{eq:objective}), we first analyze the impact of each obstacle $j$ on the trajectory planning gap.
% Specifically, we can sample a trajectory index $s_{j,t}$ according to the trajectory probability $\hat{\bm{\pi}}_{j,t}$ from vehicle $j$ and obtain the corresponding trajectory realization as $[\hat{\bm{Y}}_{j,t}']_{s_{j,t}}$.
To isolate the impact of obstacle $j$, we make the following assumptions as done in \cite{Ass11,Ass12,Ass21,Ass22}:
\begin{itemize}
\assumptionitem{Obstacles Mutual Independence}{assump:obstacle_independence}
The trajectory distributions $\hat{\bm{\pi}}_{j,t}$ of different obstacles
$j \in \mathcal{M}_t$ are mutually independent, which is given by
\begin{equation}
\Pr\!\left(\hat{\bm{Y}}_{1,t}^{s_{1,t}},\dots,\hat{\bm{Y}}_{M_t,t}^{s_{M_t,t}}\right)
  =
  \prod_{j=1}^{M_t}
  [\hat{\bm{\pi}}_{j,t}]_{s_{j,t}} .
\end{equation}

\assumptionitem{First-Order Planner Separability}{assump:planner_separability}
The trajectory planning of vehicle $i$ under distinct obstacle-mode perturbations is approximately additive, which is given by
\begin{equation}\label{eq:planner_separability}
\hat{\bm{\tau}}_{i,t}^{(p_1,\dots,p_{M_t})}
\approx
\sum_{j=1}^{M_t}
\hat{\bm{\tau}}_{i,t}^{(p_j)} ,
\end{equation}
where $\hat{\bm{\tau}}_{i,t}^{(p_1,\dots,p_{M_t})}$ is the DT-planned trajectory of vehicle $i$ while each obstacle $j$ select trajectory $p_j$.
Similarly, $\hat{\bm{\tau}}_{i,t}^{(p_j)}$ is the DT-planned trajectory under its current belief and with obstacle $j$ choosing trajectory $p_j$, and other obstacles in its mixture-mean prediction. 
Inter-obstacle interaction effects are assumed to be negligible.
\end{itemize}
% \begin{itemize}
% \item {\it{Assumption~1}:} (Obstacles' mutually independent)
% The trajectory distributions $\hat{\bm{\pi}}_{j,t}$ of different obstacles $j \in \mathcal{M}_t$ are mutually independent which is give by
% \begin{equation}
% \Pr\!\left(\hat{\bm{Y}}_{1,t}^{s_{1,t}},\dots,\hat{\bm{Y}}_{M_t,t}^{s_{M_t,t}}\right)
%   =
%   \prod_{j=1}^{M_t}
%   [\hat{\bm{\pi}}_{j,\tau}]_{s_{j,\tau}}.
% \end{equation}
% \item {\it{Assumption~2}:} (First-Order Planner Separability)
% The planner's trajectory planning for vehicle $i$ to distinct obstacles' mode perturbations is approximately additive, which is given by
% \begin{equation}\label{eq:planner_separability}
% \hat{\bm{\tau}}_{i,t}^{(p_1,\dots,p_{M_t})} \approx \sum_{j=1}^{M_t} \hat{\bm{\tau}}_{i,t}^{(j,p_j)},
% \end{equation}
% where $\hat{\bm{\tau}}_{i,t}$ is the DT-planned trajectory under its current belief (i.e., with every obstacle at its mixture-mean prediction). Inter-object interaction effects are assumed to be negligible.
% \end{itemize}

Based on these assumptions, we can analyze the effects of the obstacles on the vehicles independently.
We assume the DT generates the corresponding planned trajectory for vehicle $i$ as
\begin{equation}
    \hat{\bm{\tau}}_{i,t}^{(p_j)} 
    = f_{\mathrm{plan}}\bigl(\hat{\bm{z}}_t^{\mathrm{ctx}},\, \tilde{\bm{Y}}_{t}^{(p_j)}, \, \bm{g}_{i,t}\bigr),
\end{equation}
where $\tilde{\bm{Y}}_{t}^{(p_j)}\in \mathbb{R}^{M_t\times T_f \times 2}$ denotes the counterfactual joint trajectory set.
Specifically, in $\tilde{\bm{Y}}_{t}^{(p_j)}$, only obstacle $j$ is set to its $p$-th trajectory realization, while all other obstacles are fixed to their nominal trajectories, which is given by 
\begin{equation}
\bigl[\tilde{\bm{Y}}_{t}^{(p_j)}\bigr]_{\ell} =
  \begin{cases}
  \hat{\bm{Y}}_{j,t}^{(p_j)}, & \ell = j, \\[4pt]
  \bar{\bm{Y}}_{\ell,t}, & \ell \neq j,
  \end{cases}
\end{equation}
where
$\bar{\bm{Y}}_{\ell,t}
=
\sum_{q=1}^{P}\hat{\pi}_{\ell,t}^{q}\hat{\bm{Y}}_{\ell,t}^{(q)}$
is the nominal trajectory of obstacle $\ell$.
Similarly, we can also define the ideal trajectory of vehicle $i$ as under the full state information collection, which is given by
\begin{equation}
\bm{\tau}_{i,t}^{\ast}=\hat{\bm{\tau}}_{i,t}^{(p_{1}^{\star},\dots,p_{M_{t}}^{\star})},
\end{equation}
where $p_{j}^{\star}$ indicate the ground-truth trajectory of obstacle $j$.
Since the ground-truth trajectory of obstacle~$j$ is unavailable, DT can only request $\bm{z}_{j,t}^{\mathrm{obstacle}}$ from~\eqref{eq:prediction} that can govern its distribution.

After analyzing the obstacle's effect, we now turn to the DT's drift error.
To this end, we then make the following assumptions as done in \cite{Ass32,Ass41,Ass43,Ass51}:
\begin{itemize}
\assumptionitem{Simulator Stability}{assump:simulator_stability}
\textcolor{black}{The DT simulator $f_{\mathrm{simulate}}$ is non-expansive and there exists
$L_{\mathrm{sim}}\leq1$ such that
$\|f_{\mathrm{simulate}}(\hat{\bm{z}},\hat{\bm{Y}},\hat{\bm{z}}^{\mathrm{map}})
-
f_{\mathrm{simulate}}(\hat{\bm{z}}',\hat{\bm{Y}},\hat{\bm{z}}^{\mathrm{map}})\|_2
\leq
L_{\mathrm{sim}}\|\hat{\bm{z}}-\hat{\bm{z}}'\|_2$.}

\assumptionitem{Planner Lipschitz Continuity}{assump:planner_lipschitz}
$f_{\mathrm{plan}}$ is $L_{\mathrm{plan}}$-Lipschitz in its context embedding:
$\|f_{\mathrm{plan}}(\hat{\bm{z}}^{\mathrm{ctx}},\hat{\bm{Y}},\bm{g})
-
f_{\mathrm{plan}}(\hat{\bm{z}}^{\mathrm{ctx},\ast},\hat{\bm{Y}},\bm{g})\|_2
\leq
L_{\mathrm{plan}}
\|\hat{\bm{z}}^{\mathrm{ctx}}-\hat{\bm{z}}^{\mathrm{ctx},\ast}\|_2$.

\assumptionitem{Bounded Mode Divergence}{assump:bounded_mode_divergence}
For any two trajectory modes $p,q$, the simulation outputs differ by at most
$\delta_{\mathrm{mode}}$ as
$\|f_{\mathrm{simulate}}(\hat{\bm{z}},\hat{\bm{Y}}^{(p)},\hat{\bm{z}}^{\mathrm{map}})
-
f_{\mathrm{simulate}}(\hat{\bm{z}},\hat{\bm{Y}}^{(q)},\hat{\bm{z}}^{\mathrm{map}})\|_2
\leq
\delta_{\mathrm{mode}}$.

\assumptionitem{Bounded Compression Error}{assump:bounded_compression_error}
The cross-attention query-response cycle in~\eqref{eq:per_frame_fusion}
introduces at most $\eta_{\mathrm{QR}}$ state reconstruction error as
$\|\hat{\bm{z}}_{j,t}^{\mathrm{agg}}
-
\hat{\bm{z}}_{j,t}^{\mathrm{agg},\ast}\|_2^2
\leq
\eta_{\mathrm{QR}}$,
when $\Delta_{j,t}=0$.
\end{itemize}
% \begin{itemize}
% \item {\it{Assumption$~3$} (Simulator Stability):}
% \textcolor{black}{The DT simulator $f_{\mathrm{simulate}}$ is non-expansive: there exists $L_{\mathrm{sim}}\leq1$ such that
% $\|f_{\mathrm{simulate}}(\hat{\bm{z}},\hat{\bm{Y}},\hat{\bm{z}}^{\mathrm{map}}) - f_{\mathrm{simulate}}(\hat{\bm{z}}',\hat{\bm{Y}},\hat{\bm{z}}^{\mathrm{map}})\|_2 \leq L_{\mathrm{sim}}\|\hat{\bm{z}}-\hat{\bm{z}}'\|_2$.}

% \item {\it{Assumption$~4$} (Planner Lipschitz Continuity):}
% \textcolor{black}{$f_{\mathrm{plan}}$ is $L_{\mathrm{plan}}$-Lipschitz in its context embedding: $\|f_{\mathrm{plan}}(\hat{\bm{z}}^{\mathrm{ctx}},\hat{\bm{Y}},\bm{g}) - f_{\mathrm{plan}}(\hat{\bm{z}}^{\mathrm{ctx},\ast},\hat{\bm{Y}},\bm{g})\|_2 \leq L_{\mathrm{plan}}\|\hat{\bm{z}}^{\mathrm{ctx}}-\hat{\bm{z}}^{\mathrm{ctx},\ast}\|_2$.}

% \item {\it{Assumption$~5$} (Bounded Compression Error):}
% \textcolor{black}{The cross-attention query-response cycle in~\eqref{eq:per_frame_fusion} introduces at most $\eta_{\mathrm{QR}}$ state reconstruction error as $\|\hat{\bm{z}}_{j,t}^{\mathrm{agg}} - \hat{\bm{z}}_{j,t}^{\mathrm{agg},\ast}\|_2^2 \leq \eta_{\mathrm{QR}}$ when $\Delta_{j,t}=0$.}

% \item {\it{Assumption$~6$} (Bounded Mode Divergence):}
% \textcolor{black}{For any two trajectory modes $p,q$, the simulation outputs differ by at most $\delta_{\mathrm{mode}}$ as}
% $\|f_{\mathrm{simulate}}(\hat{\bm{z}},\hat{\bm{Y}}^{(p)},\hat{\bm{z}}^{\mathrm{map}}) - f_{\mathrm{simulate}}(\hat{\bm{z}},\hat{\bm{Y}}^{(q)},\hat{\bm{z}}^{\mathrm{map}})\|_2 \leq \delta_{\mathrm{mode}}$.
% \end{itemize}

Based on these assumptions, we can prove that the DTs' state drift is upper-bounded by Lemma~\ref{lem:drift}.
\begin{lemma}\label{lem:drift}
\textcolor{black}{Under Assumptions~3--6, the DT's state reconstruction error for an obstacle with staleness $\Delta_{j,t}$ satisfies}
\begin{equation}\label{eq:state_drift}
\;\|\hat{\bm{z}}_{j,t}^{\mathrm{agg}} - \hat{\bm{z}}_{j,t}^{\mathrm{agg},\ast}\|_2^2 \;\leq\; 2\eta_{\mathrm{QR}} + 2\Delta_{j,t}^2\delta_{\mathrm{mode}}^2,
\end{equation}
\textcolor{black}{where $\hat{\bm{z}}_{j,t}^{\mathrm{agg},\ast}$ is the idealized aggregated state under instantaneous, lossless synchronization.} %We provide the proof in the Appendix A.
\end{lemma}

\begin{proof}
    % See Appendix \ref{proof:Lemma1}.
    See Appendix A.
\end{proof}

%\vspace{4pt}
Lemma~\ref{lem:drift} reveals a clean separation that the DT's state error decomposes into an irreducible compression term $\eta_{\mathrm{QR}}$ and a linear drift term $\Delta_{j,t}^2\delta_{\mathrm{mode}}^2$ that accumulates over unqueried frames. 
Then, we can obtain the upper bound of the gap between the DT constructed via the proposed query-response state synchronization method and an idealized DT, which is shown in Theorem~\ref{thm:equivalence}.
\begin{theorem}\label{thm:equivalence}
Under \assumpref{assump:planner_lipschitz}--6 and Lemma~\ref{lem:drift}, the upper bound of the gap between the planning trajectory of the DT constructed via the proposed query-response state synchronization method and the idealized DT is given by
\begin{equation}\label{eq:decomp_identity}
\sum_{i=1}^{N}\mathbb{E}\bigl[\|\hat{\bm{\tau}}_{i, t}-\bm{\tau}_{i,t}^{\ast}\|^{2}\bigr] \;\leq\; \sum_{j=1}^{M_t} \bigl[\alpha_{j,t} + \textcolor{black}{\varepsilon(\Delta_{j,t})}\bigr],
\end{equation}
where $\alpha_{j,t}= \sum_{i=1}^{N}\sum_{p=1}^{P}\hat{\pi}_{j,t}^{p}\,\bigl\|\hat{\bm{\tau}}_{i,t}^{(p_j)} - \hat{\bm{\tau}}_{i,t}\bigr\|^{2}$ is the trajectory sensitivity that measures obstacle~$j$'s instantaneous mode uncertainty impact on other vehicles' trajectories.
$\varepsilon(\Delta_{j,t})\triangleq N L_{\mathrm{plan}}^{2}\bigl(2\eta_{\mathrm{QR}} + 2\Delta_{j,t}^{2}\cdot\delta_{\mathrm{mode}}^{2}\bigr)$ is the DT drift error bound established in Lemma~\ref{lem:drift}.
\end{theorem}
\begin{proof}
    % See Appendix~\ref{app:proof_main}.
    See Appendix~B.
\end{proof}

From Theorem~\ref{thm:equivalence}, we observe that the expected trajectory planning gap is jointly upper-bounded by two coupled factors: (i) the trajectory sensitivity $\alpha_{j,t}$, which captures the instantaneous mode uncertainty of each obstacle $j$ as well as its impact to each vehicle $i$ by DT's simultaion, and \textcolor{black}{(ii) the staleness drift $\varepsilon(\Delta_{j,t})$, which grows monotonically as the DT's belief becomes outdated. When the DT has not queried obstacle~$j$ for $\Delta_{j,t}$ frames, its internal simulation accumulates errors in that obstacle's predicted trajectory, which in turn propagates to the planned trajectories of all vehicles interacting with it.}

Then, the optimization problem~\eqref{eq:objective} can be rewritten as minimizing the per-obstacle sensitivity and drift under communication resource constraints as
\begin{equation}\label{eq:objective2}
    \begin{aligned}
    \min_{\bm{u}_t,\, \{\bm{q}_{i,j,t}\},\, \boldsymbol{A}_{t}} \sum_{t\in\mathcal{T}}\sum_{j\in\mathcal{M}_t}
    \bigl[\alpha_{j,t}+\textcolor{black}{\varepsilon(\Delta_{j,t})}\bigr],
    \end{aligned}
\end{equation}
\begin{align}\label{c2}
    \setlength{\abovedisplayskip}{-20 pt}
    \setlength{\belowdisplayskip}{-20 pt}
    {\rm{s.t.}} \,\,
    & \eqref{eq:c1}-\eqref{eq:c4}. \notag
\end{align}

Based on the simplified problem, we can further decouple the joint optimization into two steps: 
\begin{enumerate}
    \item DT should select a subset of obstacles $j$ for state synchronization (i.e., decide $u_{i,j,t}$) to minimize the sensitivity $\alpha_{j,t}$ and \textcolor{black}{drift $\varepsilon(\Delta_{j,t})$};
    \item DT should also decide the query design (i.e., optimize $\bm{q}_{i,j,t}$) while considering the communication constraint (i.e., optimize $\bm{A}_{t}$).
\end{enumerate}
This decoupling can reduce computational burden by avoiding joint search over the massive space of all possible obstacles across all vehicles.

\section{Method Design}\label{sec:query}

In this section, our goal is to propose a novel query-response algorithm that enables DT to actively minimize the per-obstacle sensitivity and drift under communication resource constraints through determining $\bm{u}_t, \{\bm{q}_{i,j,t}\}, \boldsymbol{A}_{t}$ based on its simulation.
% Building on the problem decomposition in Theorem~\ref{thm:equivalence}, the proposed method consists of five components: a K-step trajectory sensitivity estimator, a progressive query-response mechanism with spatial envelope grounding, a cross-timestep predictive query dispatch mechanism that decouples query transmission from response computation, a greedy pipeline scheduler tied to the reformulated objective~\eqref{eq:objective2}, and event-triggered birth/death notification.

\subsection{Trajectory Sensitivity and Drift Error Estimation}\label{sec:sensitivity}

The reformulated objective problem~\eqref{eq:objective2} requires the DT to select a subset of obstacles for state synchronization to minimize the sum of per-obstacle trajectory sensitivities $\alpha_{j,t}$ and drift errors $\varepsilon(\Delta_{j,t})$ under communication constraints.
However, direct evaluation of $\alpha_{j,t}$ demands $M_t \times N \times P$ planner rollouts per frame, which is prohibitive for real-time deployment. 

We address this challenge with a predictor that can estimate per-step sensitivity and drift error over the full scheduling horizon. 
Given the DT's current belief about obstacle~$j$, the predictor produces
\begin{equation}\label{eq:sensitivity_kstep}
  \hat{\bm{\mu}}_{j,t}
  = Q_\psi\bigl(\hat{\bm{z}}_{j,t}^{\mathrm{agg}},\;\hat{\bm{Y}}_{j,t},\;\hat{\bm{\pi}}_{j,t},\;\hat{\bm{z}}_{j,t}^{\mathrm{motion}}\bigr),
\end{equation}
where $\hat{\bm{\mu}}_{j,t}=
  \bigl(\hat{\mu}^{(0)}_{j,t},\dots,\hat{\mu}^{(N_{\max})}_{j,t}\bigr)$ is the predicted trajectory sensitivity $\alpha_{j,t+K}$ at $K$ steps into the future and $N_{\max}$ is the maximum prediction steps.

The predicted per-obstacle scheduling score combines the learned sensitivity with the closed-form drift bound $\varepsilon(\cdot)$ from Theorem~\ref{thm:equivalence} as
\begin{equation}\label{eq:score_future}
  \hat{\alpha}^{(K)}_{j,t} \;=\; \hat{\mu}^{(K)}_{j,t} + c_1 + c_2 \cdot (\Delta_{j,t} + K)^2,
\end{equation}
where $c_1 = 2 N L_{\mathrm{plan}}^2 \eta_{\mathrm{QR}}$ and $c_2 = 2 N L_{\mathrm{plan}}^2 \delta_{\mathrm{mode}}^2$ absorb the constants from Lemma~\ref{lem:drift}, and $\Delta_{j,t} + K$ is the total staleness at the query execution frame $t+K$. Both $c_1$ and $c_2$ are hyperparameters introduced by the DT drift and can be obtained from an offline experiment.

The predictor is trained offline on logged driving sequences via mean squared error against ground-truth sensitivity, which is given by
\begin{equation}\label{eq:critic_loss_kstep}
  \mathcal{L}_{Q_\psi} \;=\; \sum_{j=1}^{M_{t}}\sum_{K=0}^{N_{\max}}
  \bigl(\hat{\mu}^{(K)}_{j,t} - \alpha_{j,t+K}\bigr)^2.
\end{equation}

\subsection{Progressive and Cross-Timestep Query-Response Mechanism}\label{sec:query_mech}

Given the predicted sensitive and DT drift scores $\hat{\alpha}^{(K)}_{j,t}$ in (\ref{eq:score_future}), the DT must actively query vehicles to obtain fresh observations of high-scoring obstacles. 
However, two key challenges have emerged during the querying process.
First, the DT's belief may have drifted from the physical reality, leading to a mismatch between the obstacle DT needs to query and the physical obstacle in reality.
Second, the end-to-end query-response latency within a single time slot is unacceptable, which contains DT dispatching a query, the vehicle computing a response, and the DT updating its belief.

We address the first problem using a pair of progressive query vectors, which contain a coarse query $\mathcal{G}^{(K)}_{i,j,t}$ for obstacle location and a fine query $\bm{q}_{i,j,t}$ after obstacle selection. 
Then, to address the second problem, we propose a \emph{cross-timestep query dispatch} mechanism, which enables the DT to issue queries for future time steps, i.e., $t+K$, in advance, thereby avoiding query–response operations within the same time slot.
Specifically, at frame $t$, the DT may issue a predictive query bundle to vehicle~$i$ for obstacle~$j$ for the state $K$ slot letter, which is given by
\begin{equation}\label{eq:query_bundle}
  \mathcal{Q}_{i,j,t}^{(K)} = \bigl( \mathcal{G}_{i,j,t}^{(K)},\bm{q}_{i,j,t}, K\bigr),
\end{equation}
where $\mathcal{G}_{i,j,t}^{(K)}$ is the spatial envelope used as coarse query, $\bm{q}_{i,j,t}\in\mathbb{R}^{d_q}$ is the fine query vector, and $K\in[0,N_{\max}]$ is the scheduling horizon.

In particular, the position obstacle~$j$'s position from vehicle $i$'s perspective at future timestep $t+K$ can be given by DT's simulation as
% TODO:是不是和system model定义存在重复？
\begin{equation}\label{eq:y_hat}
\hat{\bm{y}}^{(K)}_{j,t} = \sum_{\bm{s}\in\{1,\dots,P\}^K} \Pr(\bm{s})\;[\hat{\bm{Y}}_{j,t+K-1}^{(\bm{s})}]_{2},
\end{equation}
where $\bm{s}=(s_{j,t},\dots,s_{j,t+K-1})$ is the $P^K$ mode sequences over $K$ unqueried frames, each with probability shown in (\ref{eq:probability}).
Then, DT generates an envelope region of obstacle $j$ from the perspective of vehicle $i$ as
\begin{equation}\label{eq:envelope_inflation}
\begin{aligned}
  \mathcal{G}^{(K)}_{i,j,t} \;=\; \bigl(
    & \|\hat{\bm{y}}^{(K)}_{j,t} - \bm{v}_{i,t}\|_2 \pm (\kappa_\rho \sqrt{c_1 + c_2 K^2} + \eta K),\\
    \;
    & \measuredangle(\hat{\bm{y}}^{(K)}_{j,t} - \bm{v}_{i,t}) \pm (\kappa_\theta \sqrt{c_1 + c_2 K^2} + \eta K)
  \bigr),
\end{aligned}
\end{equation}
where $\kappa_\rho,\kappa_\theta,\eta$ are fixed hyperparameters, and $\measuredangle(\cdot)$ denotes the polar angle. The uncertainty term $\sqrt{c_1 + c_2 K^2}$ inflates the envelope proportionally to the drift bound.
The drift term $\eta K$ adds extra margin proportional to the scheduling delay.

Since the dispatch at frame~$t$ consumes negligible downlink resources, the delay constraint~\eqref{eq:c1} applies solely to the uplink at frame~$t+K$. The DT must ensure that the predicted uplink delay for vehicle~$i$ at the future frame $t+K$ satisfies
\begin{equation}\label{eq:predictive_delay}
  \hat{\mathcal{T}}\!({\boldsymbol{A}_{t+K}},\bm{v}_{i,t+K})
  % \hat{\tau}_{i,t+K} 
  \!\!=\!\! \frac{C}{A_{i,t+K}B_w\log_2\!\Bigl(1 \!\!+\!\! \frac{P_{\mathrm{tx}}}{B_w N_0\mathrm{PL}(\hat{d}_{i,t+K})}\Bigr)},
\end{equation}
where $\hat{d}_{i,t+K} = \|\hat{\bm{v}}_{i,t+K} - \bm{p}_{\mathrm{server}}\|_2$ is the predicted 3D distance between vehicle~$i$ and the DT server, obtained from the DT's own trajectory prediction of vehicle~$i$ via the planning function~$\hat{\bm{\tau}}_{i,t}$ in~\eqref{eq:dt_planning}.

Then, we show how the progressive query is trained with separate objectives corresponding to its distinct roles offline.

\emph{Stage~1 (Coarse Grounding).} The coarse query that embodied by the envelope $\mathcal{G}^{(K)}_{i,j,t}$ and the DT's position prediction $\hat{\bm{y}}^{(K)}_{j,t}$ that is trained to minimize the spatial prediction error between the DT's expected position and the obstacle's ground-truth location as
\begin{equation}\label{eq:loss_stage1}
  \mathcal{L}_{\mathrm{coarse}} \;=\; \mathbb{E}_{(j,t,K)}\!\Bigl[\bigl\|\bm{p}^{\mathrm{gt}}_{j,t+K} - \hat{\bm{y}}^{(K)}_{j,t}\bigr\|_2^2\Bigr],
\end{equation}
where $\bm{p}^{\mathrm{gt}}_{j,t+K}$ is the ground-truth position of obstacle~$j$ at frame $t+K$. Minimizing $\mathcal{L}_{\mathrm{coarse}}$ penalizes belief drift when $\hat{\bm{y}}^{(K)}_{j,t}$ deviates from the true position, the envelope fails to cover the correct detection, causing a matching failure.

\emph{Stage~2 (Fine Extraction).} The fine query $\bm{q}_{i,j,t}$ is trained to jointly minimize mode-classification error and the post-query trajectory sensitivity as
\begin{equation}\label{eq:loss_stage2}
\begin{aligned}
  \mathcal{L}_{\mathrm{fine}} \;&=\;
  \mathbb{E}_{(j,t)}\!\Bigl[
    \underbrace{-\log\frac{\exp\!\bigl(\langle\bm{W}_e\hat{\bm{z}}^{\mathrm{motion}}_{j,t}[p^{\star}_{j,t}],\; \bm{r}_{i,j,t}\rangle / \tau\bigr)}
            {\sum_{p=1}^{P}\exp\!\bigl(\langle\bm{W}_e\hat{\bm{z}}^{\mathrm{motion}}_{j,t}[p],\; \bm{r}_{i,j,t}\rangle / \tau\bigr)}}_{\mathcal{L}_{\mathrm{mode}}:\text{mode identification}}
    \;\\
    &+\; \lambda
    \underbrace{\sum_{i=1}^{N}\sum_{p=1}^{P} \hat{\pi}^{\mathrm{post}}_{j,p}
    \bigl\|\hat{\bm{\tau}}_{i,t}^{(j,p)} - \hat{\bm{\tau}}_{i,t}\bigr\|^2}_{\mathcal{L}_{\mathrm{sens}}:\text{post-query sensitivity}}
  \Bigr],
\end{aligned}
\end{equation}
where $\lambda>0$ balances mode identification against sensitivity reduction. The posterior $\hat{\pi}^{\mathrm{post}}_{j,p}$ is obtained via a differentiable TokenBayes update, and the gradient back-propagates through the response $\bm{r}_{i,j,t}$ to update the query parameters.
The proposed algorithm is shown in Fig. \ref{fig_algorithm}. 
\begin{figure}[t]
\includegraphics[width=\linewidth]{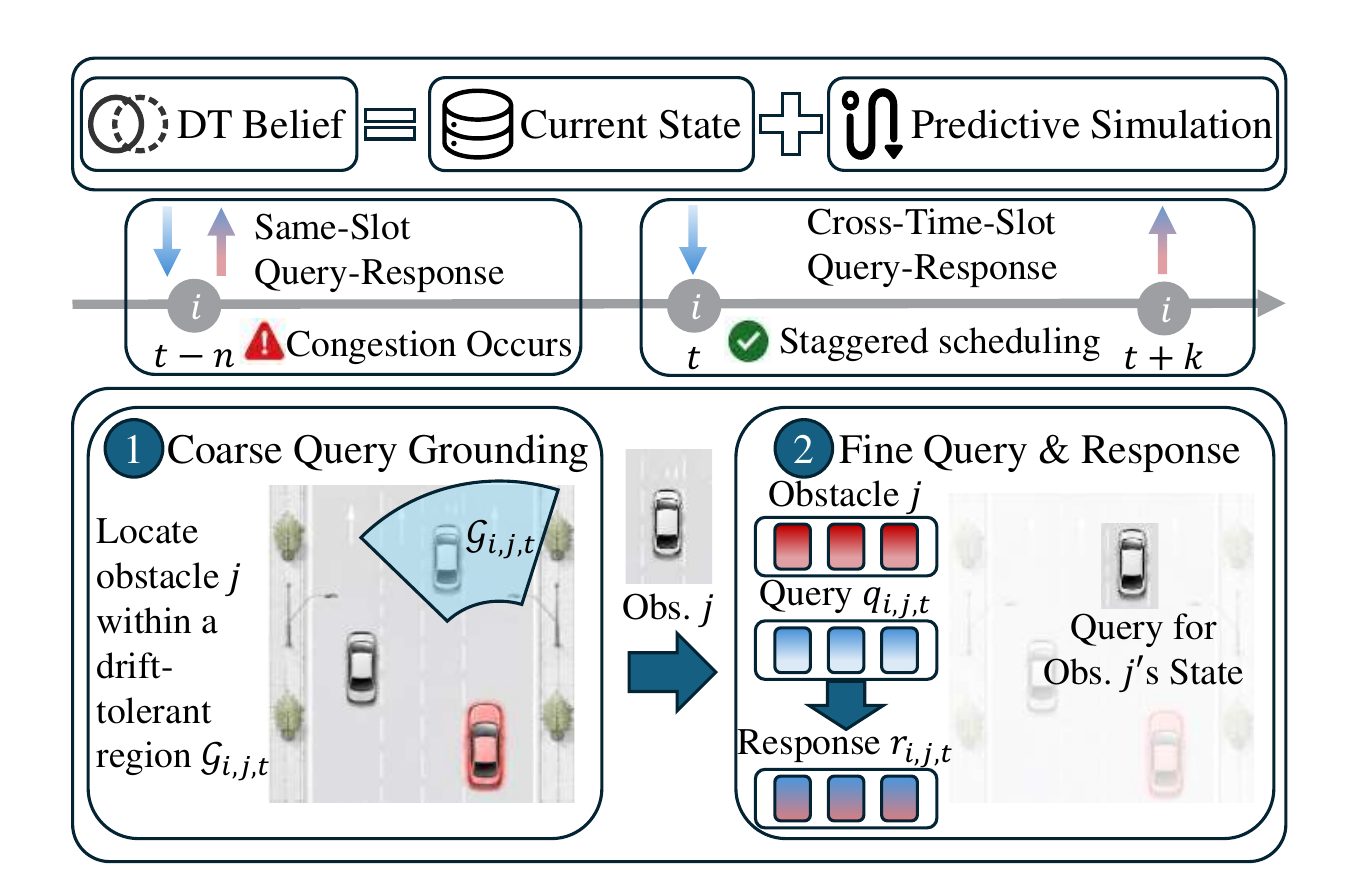}
\caption{Illustration of the proposed progressive and cross-timestep query-response mechanism.}
\label{fig_algorithm}
\end{figure}
\subsection{Joint Obstacle--Vehicle--Horizon Scheduling}
\label{sec:scheduling}

At each frame $t$, the DT jointly determines three coupled scheduling decisions: 
(i) which obstacles should be queried, (ii) which future horizon should be used for collecting the corresponding responses, and (iii) which vehicles should serve as observers. 
This joint design is necessary because an obstacle with a large planning impact may not be immediately observable by suitable vehicles, while deferring its response to a future frame also changes both the observer set and the available communication resources.

To support predictive query dispatch, we first estimate the set of vehicles that are likely to observe obstacle $j$ at a candidate future frame $t+K$ as
\begin{equation}
\label{eq:vis_hat}
\hat{\mathcal{U}}_{j,t+K}
=
\left\{
i\in\mathcal{N}
\mid
\left\|
\hat{\bm{v}}_{i,t+K}
-
\hat{\bm{y}}^{(K)}_{j,t}
\right\|_2
\le d_{\mathrm{det}}
\right\},
\end{equation}
where $\hat{\bm{y}}^{(K)}_{j,t}$ denotes the DT-predicted position of obstacle $j$ after $K$ frames, $\hat{\bm{v}}_{i,t+K}$ is obtained from the predicted trajectory of vehicle $i$, and $d_{\mathrm{det}}$ is the detector range. 
For each admitted vehicle--obstacle pair, the DT sends a query of size $C_q$ bits at frame $t$, and the selected vehicle uploads a response of size $C_r$ bits at frame $t+K$.

We next characterize the minimum RB requirement for a vehicle to transmit a given payload within the latency constraint. 
Let $\mathcal{J}\subseteq\mathcal{M}_{t'}$ be a generic obstacle set with per-obstacle payload size $C$, which enters only through the aggregate payload $|\mathcal{J}|C$. The minimum number of RBs for vehicle $i$ to transmit this payload at frame $t'$ within the latency constraint is defined as
\begin{equation}
\label{eq:amin}
A_i^{\min}\!\left(\mathcal{J};C,t'\right)
\!\!=\!\!
\min
\left\{
A\in(0,R]:
\mathcal{T}\left(A,\bm{v}_{i,t'},|\mathcal{J}|C\right)
\le
\mathcal{T}_{\max}
\right\}.
\end{equation}
Since the transmission latency decreases monotonically with the RB allocation, the above minimum is well-defined whenever the latency constraint is feasible.

Based on this RB requirement, we design a priority-aware greedy scheduler. 
Obstacles are first sorted in descending order of their predicted current-frame planning impact $\hat{\alpha}^{(0)}_{j,t}$ in~\eqref{eq:score_future}. 
For each obstacle $j$, the scheduler scans candidate horizons $K=0,1,\ldots,N_{\max}$ and attempts to admit at most $K_v$ observing vehicles from $\hat{\mathcal{U}}_{j,t+K}$. 
Let $S_j$ denote the set of vehicles already admitted for obstacle $j$. 
For a candidate observer $i\in\hat{\mathcal{U}}_{j,t+K}\setminus S_j$, the additional downlink query RB cost at frame $t$ is
\begin{equation}
\label{eq:rho_q}
\rho^{\mathrm{q}}_{i}(j)
=
A_i^{\min}\!\left(
\mathcal{J}_i^{\mathrm{q}}\cup\{j\};C_q,t
\right)
-
A_{i,t}^{\mathrm{q}},
\end{equation}
and the additional uplink response RB cost at frame $t+K$ is
\begin{equation}
\label{eq:rho_r}
\rho^{\mathrm{r}}_{i}(j,K)
=
A_i^{\min}\!\left(
\mathcal{J}_i^{\mathrm{r},(K)}\cup\{j\};C_r,t+K
\right)
-
A_{i,t+K}^{\mathrm{r}},
\end{equation}
where $\mathcal{J}_i^{\mathrm{q}}=\{j'\mid u_{i,j',t}=1\}$ and $\mathcal{J}_i^{\mathrm{r},(K)}=\{j'\mid u_{i,j',t}=1,\,K_{j'}^\ast=K\}$ denote, respectively, the sets of obstacles already assigned to vehicle $i$ as query and response loads earlier in the current greedy pass, with $K_{j'}^\ast$ the horizon committed to obstacle $j'$ by the scheduler.

The scheduler selects the candidate observer with the smallest marginal RB cost, which is given by
\begin{equation}
\label{eq:observer_select}
i^\ast
=
\arg\min_{i\in\hat{\mathcal{U}}_{j,t+K}\setminus S_j}
\left[
\rho^{\mathrm{q}}_{i}(j)
+
\rho^{\mathrm{r}}_{i}(j,K)
\right].
\end{equation}
The selected observer is admitted only if both the current-frame query budget and the future-frame response budget remain feasible.
$u_{i^\ast,j,t}=1$ hold if and only if $\sum_{i'} A_{i',t}^{\mathrm{q}}+
\rho^{\mathrm{q}}_{i^\ast}(j)\le B_t$ and $\sum_{i'} A_{i',t+K}^{\mathrm{r}}+\rho^{\mathrm{r}}_{i^\ast}(j,K)\le B_{t+K}$, and $u_{i^\ast,j,t}=0$, otherwise.
The admission process for obstacle $j$ terminates when $K_v$ observers have been admitted.

\section{Experiment}
For our simulation, we evaluate the proposed query-driven DT on the V2X-Sim~\cite {V2XSim} and nuScenes~\cite {fong2021panoptic} datasets.
The V2X-Sim dataset provides $5$ vehicles with $100$ frames each, including multi-camera RGB and LiDAR data, as well as the obstacles' trajectories and the vehicles' ground-truth boxes across $5$ scenarios. 
In our experiment, each vehicle runs the UniAD pipeline~\cite{hu2023uniad} on camera input for obstacle detection and simulation, as well as for ego trajectory planning.
The DT maintains a world model that predicts obstacle states following the UniAD pipeline based on the proposed queries-response method.
We select $5$ baselines for comparison. 
\begin{enumerate}
\item Semantic communication-based method~\cite{xu2021wireless}: It enables vehicles to compress all the detected obstacles' representations and transmit them to the DT server through a semantic encoder. Then, DT reconstructs the transmitted semantic information from a semantic decoder (labeled "Semantic" in plots).
\item Age of information (AoI) based method~\cite{zhou2026lowoverhead_dt}: It only leverages vehicles to transmit the obstacles' representation that has not been transmitted $T$ frames (labeled "AoI" in plots).
\item V2VNet based method~\cite{wang2020v2vnet}: It adapts the V2V feature-sharing protocol to a centralized architecture.
For this method, every vehicle uploads all the agent-query tokens for its detected obstacles to the DT server every frame (labeled "V2VNet" in plots).
\item Where2comm method~\cite{hu2022where2comm}: This method enables vehicles to upload only the top-$50\%$ obstacles ranked by detection confidence (labeled "Where2comm" in plots).
\item CoopTrack method~\cite{Cooptrack}: It employs a fully instance-level end-to-end framework for cooperative 3D multi-object tracking (labeled "CoopTrack" in plots).
\end{enumerate}

\subsection{Performance Evaluation}
\begin{figure}[t]
\includegraphics[width=\linewidth]{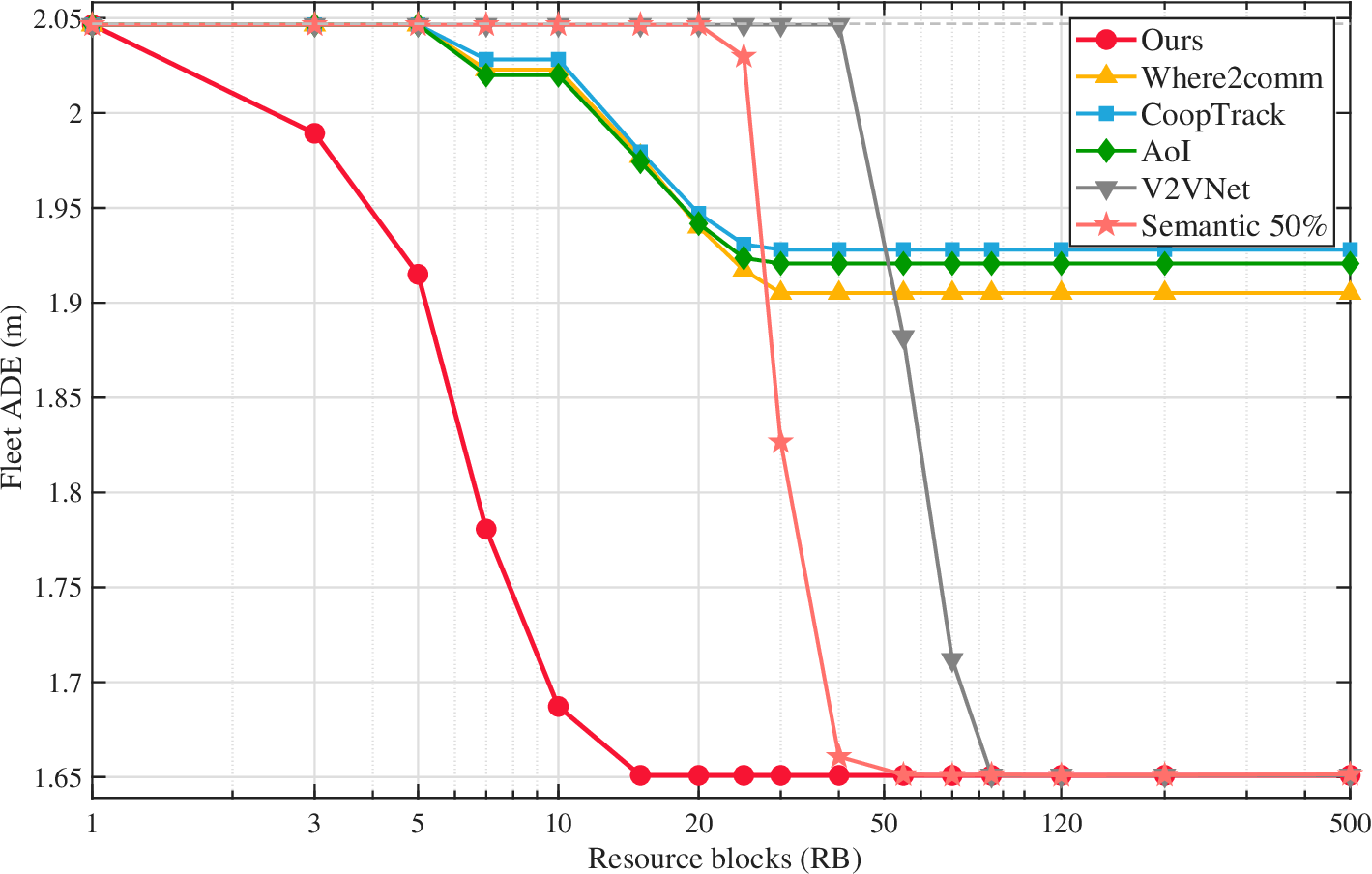}
\caption{RB v.s. Fleet ADE.}
\label{exp:fig1}
\end{figure}
Fig.~\ref{exp:fig1} shows how fleet average displacement error (ADE), which measures the trajectory error, changes as the number of RBs varies from $1$ to $500$ per frame.
From Fig.~\ref{exp:fig1}, we observe that each method's ADE decreases as the RB increases.
This is due to the fact that as the number of RBs increases, more states are uploaded at each frame, and the ADE decreases accordingly.
From Fig.~\ref{exp:fig1} we can also observe that the proposed method has the lowest ADE compared to other baselines.
This is because the proposed method enables DT to actively request the obstacles' state with the highest importance based on trajectory optimization analysis.
What's more, the query-response mechanism has much lower transmission overhead than raw-state transmission because the size of the query-response pair is much smaller than the raw-state size.
This figure also shows that the proposed method can achieve $24\%$ and $18\%$ reductions in ADE compared to the semantic method and AoI-based methods, respectively. 
This is because the proposed method offers the most flexible information-acquisition mode and can fully leverage DT’s deductive capabilities to reduce positional error in obstacles.

\begin{figure}[t]
\includegraphics[width=\linewidth]{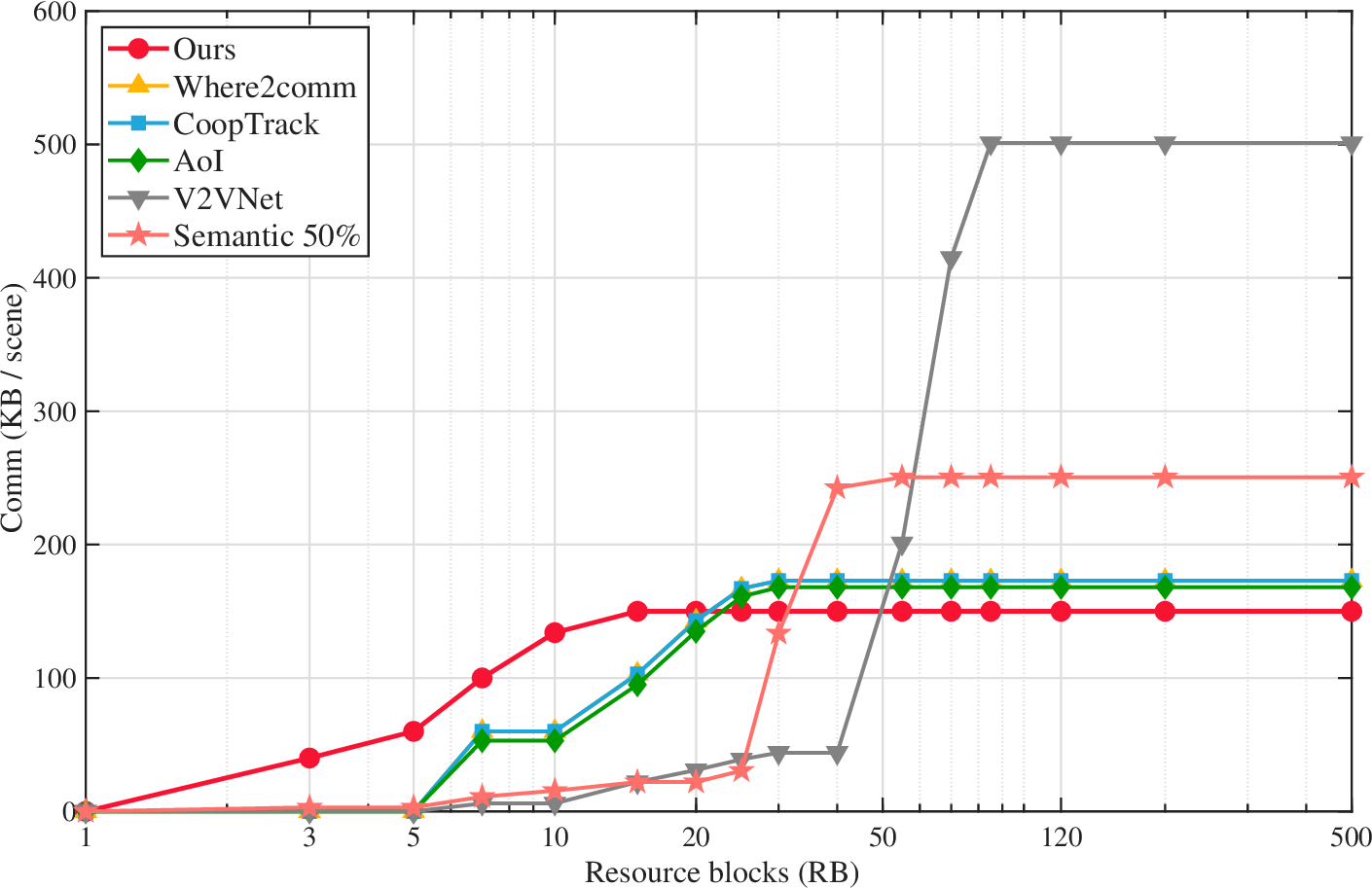}
\caption{RB v.s. Communication amount.}
\label{exp:fig2}
\end{figure}

Fig. \ref{exp:fig2} shows how the communication amount changes as the number of RBs varies.
From Fig. \ref{exp:fig2}, we can observe that as the number of RBs increases, the communication amount increases.
This is because vehicles begin uploading their status information only when the available communication resources are sufficient to support the overhead.
From this figure, we can also see that at a lower number of RBs, the proposed method shows a sharp increase in communication volume, while the baselines remain unchanged.
This is due to the fact that the proposed scheme can effectively adapt to changes in the available RBs, thereby making full use of the limited number of RBs to achieve efficient information synchronization.
From Fig. \ref{exp:fig2}, we can also observe that the proposed method achieves the lowest communication volume when sufficient RB resources are available.
This is because the proposed method fully exploits the DT's simulation capabilities and only queries partial data on demand, thereby effectively reducing communication overhead.

\begin{figure}[t]
\includegraphics[width=\linewidth]{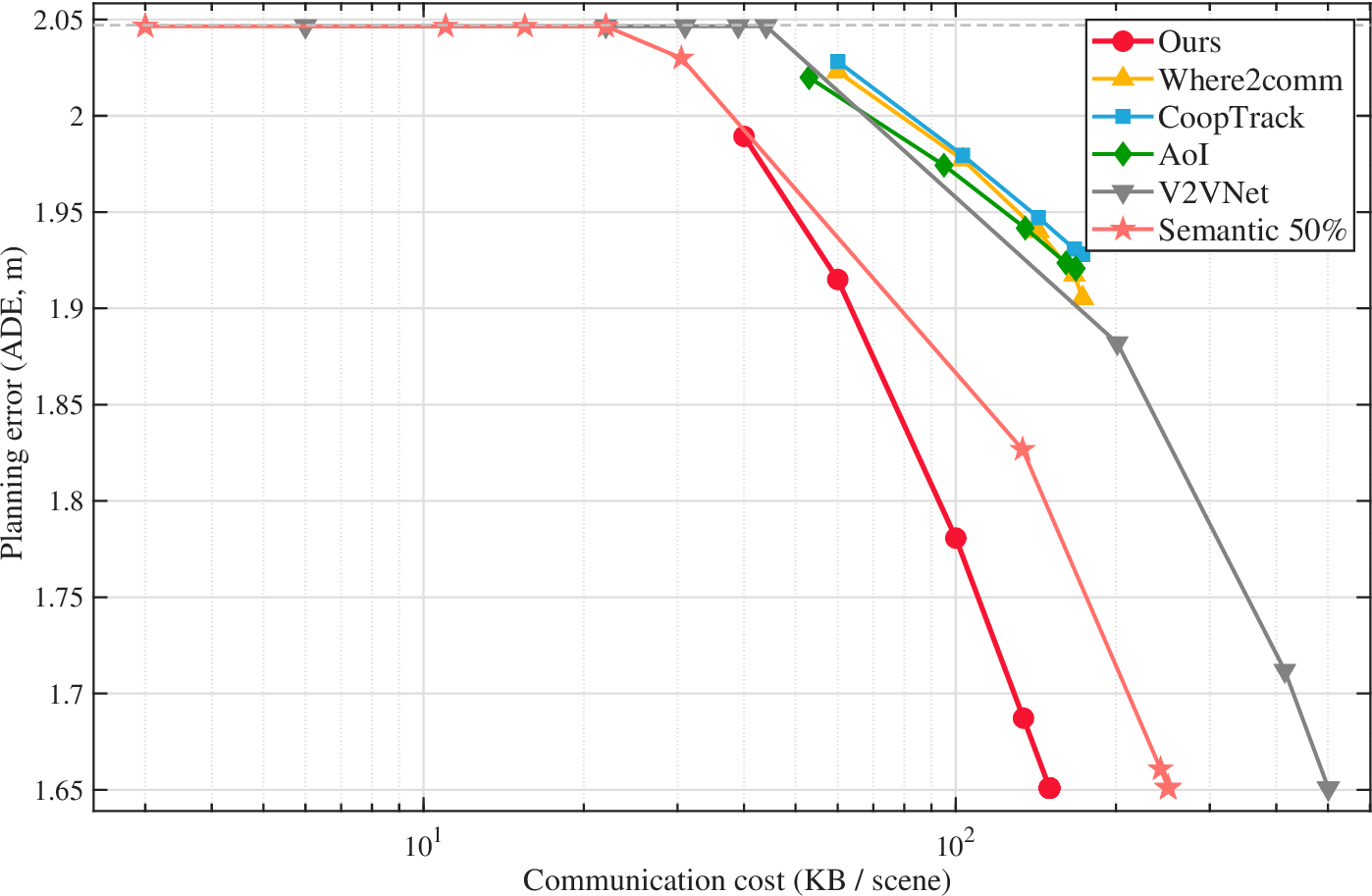}
\caption{RB v.s. Communication amount.}
\label{exp:fig3}
\end{figure}
Fig. \ref{exp:fig3} shows the Pareto trade-off of the planning ADE achieved by each scheme at every RB operating point.
From Fig. \ref{exp:fig3}, we can observe that the proposed scheme dominates the Pareto frontier.
The DT-side query schedule, combined with cross-ego token deduplication, concentrates available bandwidth on the highest-$\alpha$ obstacles across the fleet. As a result, the proposed method saturates at ADE = 1.65 m using only $150$ KB/scene, defining the minimum-comm corner of the achievable region.
Fig. \ref{exp:fig3} also shows that the push-style baselines reach the same accuracy ceiling but spend 1.7–3.3$\times$ more bandwidth. 
Their low-RB segment reveals a bandwidth-waste artifact: the per-vehicle all-or-nothing transmission protocol allows vehicles with small detection counts to broadcast while larger-load vehicles fail, producing non-zero comm without any ADE improvement. 

\begin{figure}[t]
\includegraphics[width=\linewidth]{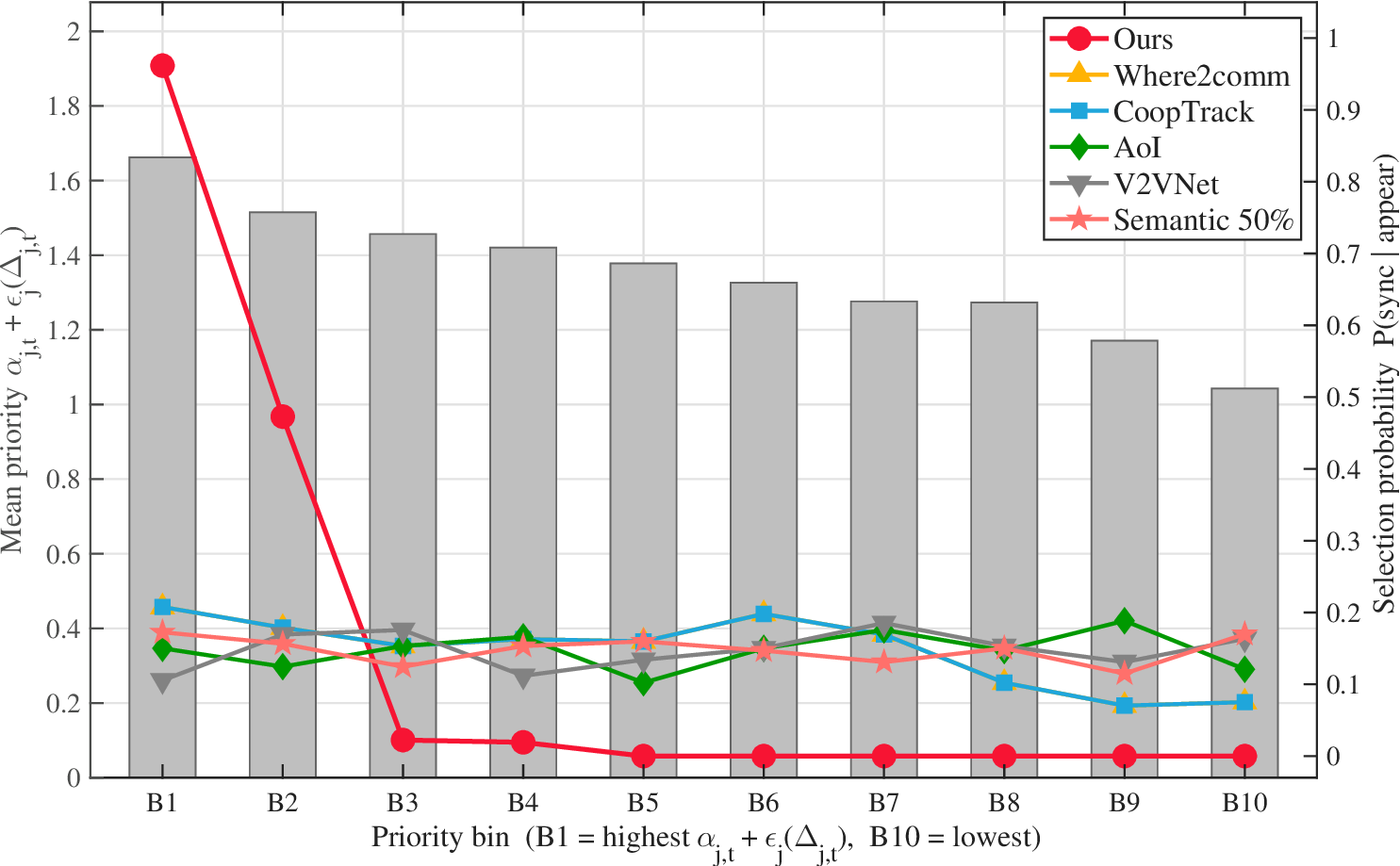}
\caption{Per-obstacle sensitivity $\hat{\alpha}_{j,t}+\varepsilon(\Delta_{j,t})$ v.s. sync frequency.}
\label{exp:fig4}
\end{figure}
Fig.~\ref{exp:fig4} shows how the selective probability change as $\hat{\alpha}_{j,t}+\varepsilon(\Delta_{j,t})$ varies at $30$ RB over $88$ frames. 
From Fig.~\ref{exp:fig4}, we can observe that the proposed method selects the obstacle with the largest $\hat{\alpha}_{j,t}+\varepsilon(\Delta_{j,t})$ while the other baselines select the obstacle average.
It is worth noting that, except for AoI, V2VNet, and the Semantic algorithm, which do not require obstacle selection, the algorithm that selects a subset of obstacles also wastes their communication resources, and the selected obstacles have little effect on the planning error.
This figure also shows that the proposed method only queries the obstacle's state when $\hat{\alpha}_{j,t}+\varepsilon(\Delta_{j,t})$ exceeds $20\%$ under the communication resource constraint.
Furthermore, due to the proposed progressive and cross-timestep query–response mechanism, the proposed method can better schedule communication resources across time slots, thereby achieving more efficient information acquisition.

\begin{figure}[t]
\includegraphics[width=\linewidth]{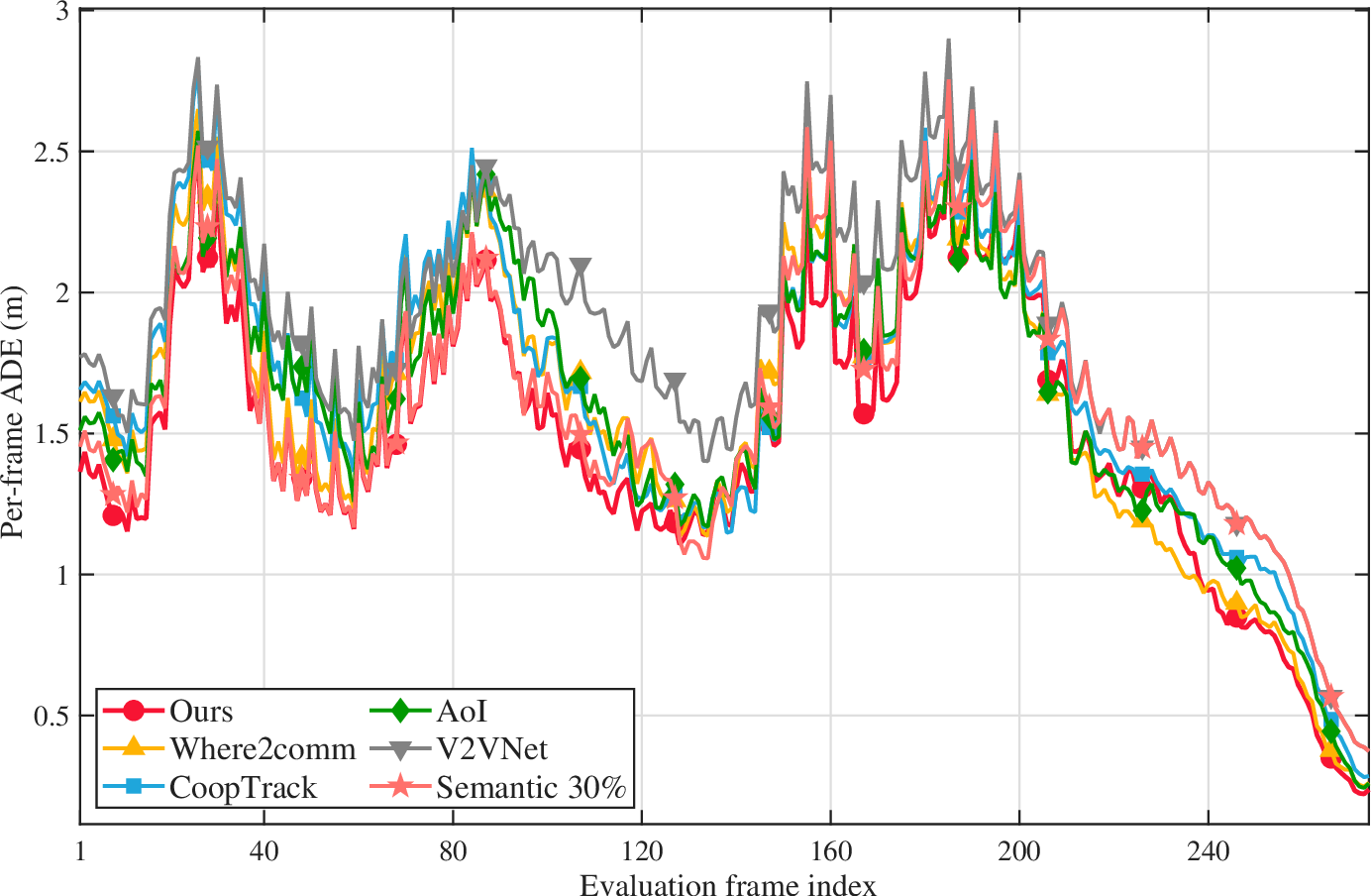}
\caption{Evaluation frame index v.s. Planning error}
\label{exp:fig5}
\end{figure}
Fig. \ref{exp:fig5} shows how the planning error changes over time steps in a continuous driving scenario. As observed from this figure, all schemes exhibit temporal fluctuations in planning error. This is because both scene complexity and driving difficulty vary continuously over time, leading to time-varying uncertainty in obstacle prediction and trajectory planning.
From Fig. \ref{exp:fig5} we can also observe that the proposed scheme consistently achieves the lowest planning error. 
This is because our scheme can effectively identify the obstacles that have the greatest impact on ADE from the perspective of DT simulation, and then precisely query their states to further reduce the simulation-induced planning error.
In contrast, the other schemes lack an effective obstacle selection mechanism and thus tend to query less informative states, resulting in larger planning errors.
We can further observe from Fig. \ref{exp:fig5} that the proposed scheme maintains a relatively stable advantage over the baselines across different time steps. This demonstrates that the proposed query-driven mechanism is not only effective in isolated scenarios but also robust under continuously changing driving conditions.

\begin{figure}[t]
\centering
\includegraphics[width=\linewidth]{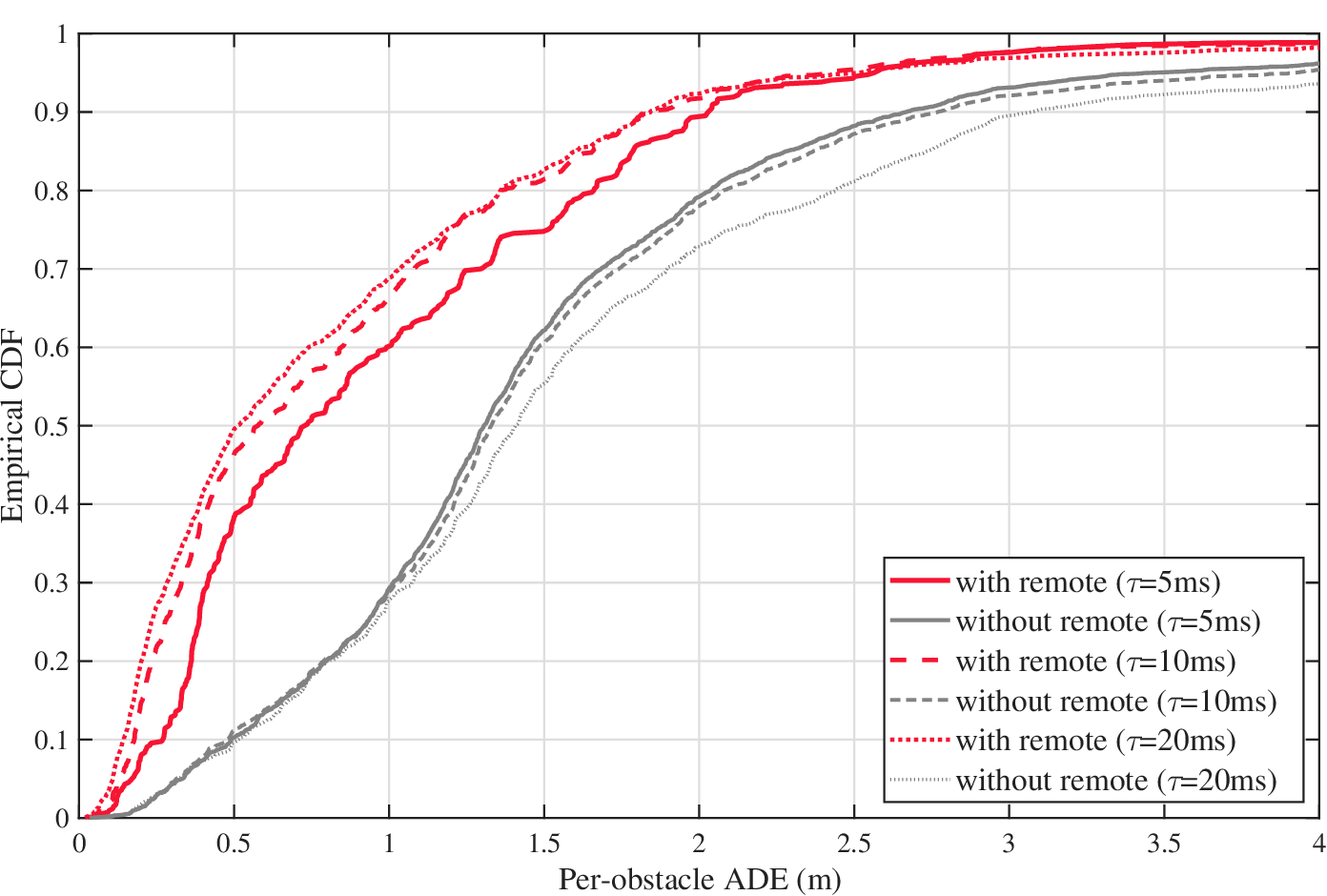}
\caption{Per-obstacle ADE empirical CDF}
\label{exp:fig6}
\end{figure}

Fig.~\ref{exp:fig6} shows the per-obstacle ADE distribution of the proposed scheme, partitioning the obstacles by whether a Stage-2 remote token arrived at the DT within the deadline. 
From Fig.~\ref{exp:fig6}, we observe that at every $\tau$ the with-remote curve lies clearly to the left of the without-remote curve, i.e., a single fine query roughly halves the per-obstacle ADE. 
This is because the response refreshes the obstacle's belief at the dispatched horizon and cancels the staleness term $\varepsilon(\Delta_{j,t})$ that otherwise accumulates over time.
We can further observe that, as the deadline budget $\tau$ grows from $5$ to $20$ms, the gap between the two curves widens.
This is because the proposed scheduler allocates the additional budget to obstacles where a fresh response yields the largest reduction in planning gap, rather than spreading queries indiscriminately.

\subsection{Ablation Study}\label{sec:ablation}
\begin{table}[t]
\centering
\caption{Module Ablation Result.}
\label{tab:ablation}
\begin{tabular}{lcccc}
\hline
\textbf{Configuration} & \textbf{ADE(m)} & \textbf{$\Delta$ADE} & \textbf{Cov.} & \textbf{$\mu_{\mathrm{wr}}$(m)} \\
\hline
Full Ours (proposed) & \textbf{1.457} & --- & \textbf{68.2\%} & \textbf{0.885} \\
$-$ no $Q_\psi$ (random) & 1.798 & $+23.4\%$ & 54.4\% & 0.919 \\
$-$ no $Q_\psi$ (entropy) & 1.606 & $+10.2\%$ & 64.4\% & 0.953 \\
$-$ no cross-attention & 1.954 & $+34.1\%$ & 72.7\% & 1.434 \\
$-$ no DT scheduling & 1.789 & $+22.8\%$ & 0.0\% & --- \\
\hline
\end{tabular}
\end{table}
To verify the contribution of each module of the proposed method, we ablate the system by replacing one module at a time with a minimal substitute as shown in Table~\ref{tab:ablation}.
In the table, coverage (Cov.) denotes the fraction of obstacles for which a Stage-2 remote token reaches the DT within the latency deadline, and $\mu_{\mathrm{wr}}$ is the mean per-obstacle ADE over these with-remote obstacles.
From Table~\ref{tab:ablation}, we observe that every module contributes substantially. Removing the cross-attention fusion head is the single largest degradation, raising ADE up to $34.1\%$ compared to the full methods, as the DT can no longer fuse multi-vehicle observations into a coherent belief.
Replacing the sensitivity predictor $Q_\psi$ with a random selector raises ADE to $23.4\%$, confirming that $Q_\psi$ is a major lever rather than a marginal refinement, whereas a mode-entropy surrogate recovers only part of this gain.
Disabling the DT-side scheduling raises ADE to $1.789$m and collapses coverage to $0\%$, since no remote token can meet the latency deadline without scheduling.
Coverage moves in tandem with $Q_\psi$, it pushes the queried subset from $54\%$ to $68\%$ of obstacles while keeping the queried-subset mean ADE the lowest, indicating that $Q_\psi$ both expands and refines the selection.

\section{Conclusion}\label{sec:conclusion}
In this paper, we propose a query-driven DT architecture in which the DT actively requests the desired state information according to its simulation result via a query-response mechanism.
To achieve this goal, we formulate an optimization problem to minimize the planning position error gap between the DT constructed via the proposed query-response method and an idealized DT.
To tackle this intractable problem, we derive an upper bound of the planning position error, expressed as the sum of per-obstacle trajectory sensitivities and a drift error that grows with information staleness. 
% Through mathematical analysis, we found that 
% Through mathematical analysis, we found that: 1) hiding critical information from neighbors can effectively prevent private data from leakage; 2) devices can verify whether the models received from neighbors have been updated as expected to identify Byzantine adversaries; 3) the convergence speed of DFL can be modeled as the characteristics of the device connection matrix which are affected by device connection scheme design.
Then, we propose a progressive, cross-timestep query-response method that dispatches coarse and fine queries to various time steps.
Finally, we proposed a joint obstacle-vehicle-horizon scheduling method that can allocate communication resources to minimize the vehicle planning error.
Experiments on multi-vehicle autonomous driving datasets demonstrate that the proposed method significantly improves trajectory planning performance under communication resource constraints.

\def\baselinestretch{1.0}
\bibliographystyle{IEEEtran}
\bibliography{references}
\newpage
\appendix
\subsection{Proof  of Lemma \ref{lem:drift}}\label{proof:Lemma1}
% ==================================================================

% We now state and prove the key lemma that bounds the DT's state reconstruction error as a function of staleness.

% \begin{lemma}[State Drift Bound]
% Let obstacle~$j$ have staleness $\Delta_{j,t}=t-\max\{\tau\leq t\mid\sum_{i\in\mathcal{U}_{j,\tau}}u_{i,j,\tau}>0\}$ as defined in~\eqref{eq:staleness}.
% Under Assumptions~\ref{ass:compression}--\ref{ass:mode_divergence} and the non-expansiveness condition $L_{\mathrm{sim}}\leq1$ of Assumption~\ref{ass:lipschitz}, the squared state reconstruction error satisfies
% \begin{equation}\label{eq:state_drift_bound}
% \|\hat{\bm{z}}_{j,t}^{\mathrm{agg}} - \hat{\bm{z}}_{j,t}^{\mathrm{agg},\ast}\|_2^2 \;\leq\; \eta_{\mathrm{QR}} + \Delta_{j,t}\cdot\delta_{\mathrm{mode}}^2,
% \end{equation}
% where $\hat{\bm{z}}_{j,t}^{\mathrm{agg},\ast}$ is the idealized aggregated state that the DT would possess under instantaneous, lossless synchronization of obstacle~$j$.
% \end{lemma}

\begin{proof}\label{proof:Lemma1}
Assume the bound holds at frame $\tau$ for staleness $\Delta_{j,\tau}=k$.
Consider frame $\tau+1$ where obstacle~$j$ not been queried ($\Delta_{j,\tau+1}=k+1$).
Let $p^{\ast}$ denote the (unknown) ground-truth motion mode of obstacle~$j$ at frame $\tau$, and let $p$ denote the mode actually used by the DT simulator (drawn from its belief $\hat{\bm{\pi}}_{j,\tau}$).
The DT propagates its belief via~\eqref{eq:mode_conditioned_belief}:
\begin{equation}
\hat{\bm{z}}_{j,\tau+1}^{\mathrm{agg}} = f_{\mathrm{simulate}}\bigl(\hat{\bm{z}}_{j,\tau}^{\mathrm{agg}},\hat{\bm{Y}}_{j,\tau}^{(p)},\hat{\bm{z}}^{\mathrm{map}}\bigr).
\end{equation}
The idealized state would propagate with the ground-truth mode:
\begin{equation}
\hat{\bm{z}}_{j,\tau+1}^{\mathrm{agg},\ast} = f_{\mathrm{simulate}}\bigl(\hat{\bm{z}}_{j,\tau}^{\mathrm{agg},\ast},\hat{\bm{Y}}_{j,\tau}^{(p^{\ast})},\hat{\bm{z}}^{\mathrm{map}}\bigr).
\end{equation}

We decompose the one-step error via the triangle inequality:
\begin{equation}\label{eq:inductive_decomp}
\begin{aligned}
&\|\hat{\bm{z}}_{j,\tau+1}^{\mathrm{agg}} - \hat{\bm{z}}_{j,\tau+1}^{\mathrm{agg},\ast}\|_2
\\\leq& \underbrace{\|f_{\mathrm{simulate}}(\hat{\bm{z}}_{j,\tau}^{\mathrm{agg}},\hat{\bm{Y}}_{j,\tau}^{(p)},\hat{\bm{z}}^{\mathrm{map}}) \!\!-\!\! f_{\mathrm{simulate}}(\hat{\bm{z}}_{j,\tau}^{\mathrm{agg}},\hat{\bm{Y}}_{j,\tau}^{(p^{\ast})},\hat{\bm{z}}^{\mathrm{map}})\|_2}_{\text{(i) mode mismatch}}\\
+ &\underbrace{\|f_{\mathrm{simulate}}(\hat{\bm{z}}_{j,\tau}^{\mathrm{agg}},\hat{\bm{Y}}_{j,\tau}^{(p^{\ast})},\hat{\bm{z}}^{\mathrm{map}})\!\!-\!\!f_{\mathrm{simulate}}(\hat{\bm{z}}_{j,\tau}^{\mathrm{agg},\ast},\hat{\bm{Y}}_{j,\tau}^{(p^{\ast})},\hat{\bm{z}}^{\mathrm{map}})\|_2}_{\text{(ii) state propagation}}.
\end{aligned}
\end{equation}

For term (i), we substituting \assumpref{assump:bounded_mode_divergence} and \assumpref{assump:simulator_stability} into~\eqref{eq:inductive_decomp}, we have
\begin{equation}
\|\hat{\bm{z}}_{j,\tau+1}^{\mathrm{agg}} - \hat{\bm{z}}_{j,\tau+1}^{\mathrm{agg},\ast}\|_2 \leq \delta_{\mathrm{mode}} + \|\hat{\bm{z}}_{j,\tau}^{\mathrm{agg}} - \hat{\bm{z}}_{j,\tau}^{\mathrm{agg},\ast}\|_2.
\end{equation}

Let $e_k \triangleq \|\hat{\bm{z}}_{j,\tau}^{\mathrm{agg}} - \hat{\bm{z}}_{j,\tau}^{\mathrm{agg},\ast}\|_2$ with staleness $k=\Delta_{j,\tau}$.
The recurrence $e_{k+1}\leq e_k + \delta_{\mathrm{mode}}$ with base case $e_0\leq\sqrt{\eta_{\mathrm{QR}}}$ (from $\Delta=0$ case) yields:
\begin{equation}
e_k \leq \sqrt{\eta_{\mathrm{QR}}} + k\cdot\delta_{\mathrm{mode}}.
\end{equation}

Squaring both sides and applying the elementary inequality $(a+b)^2\leq 2a^2+2b^2$,
\begin{equation}
\begin{aligned}
\|\hat{\bm{z}}_{j,t}^{\mathrm{agg}} - \hat{\bm{z}}_{j,t}^{\mathrm{agg},\ast}\|_2^2
&\leq 2\eta_{\mathrm{QR}} + 2\Delta_{j,t}^2\delta_{\mathrm{mode}}^2.
\end{aligned}
\end{equation}
This completes the proof.
\end{proof}

% \textbf{Remark.} The linear dependence on $\Delta_{j,t}$ in~\eqref{eq:state_drift_bound} reflects the fact that the DT simulator's error accumulates additively over unqueried frames. When a fresh query-response cycle resets $\Delta_{j,t}$ to zero, the bound collapses to the irreducible compression error $\eta_{\mathrm{QR}}$.

% ==================================================================
\subsection{Proof of Theorem~\ref{thm:equivalence}}\label{app:proof_main}
% ==================================================================

% We now prove the main decomposition:
% \begin{equation}\label{eq:goal_restated}
% \sum_{i=1}^{N}\mathbb{E}\bigl\|\hat{\bm{\tau}}_{i,t}-\bm{\tau}_{i,t}^{\ast}\bigr\|_2^{2} \;\leq\; \sum_{j=1}^{M_t}\bigl[\alpha_{j,t} + \varepsilon(\Delta_{j,t})\bigr],
% \end{equation}
% where $\varepsilon(\Delta)\triangleq N L_{\mathrm{plan}}^{2}(\eta_{\mathrm{QR}}+\Delta\cdot\delta_{\mathrm{mode}}^{2})$.
\begin{proof}
Let $p_j^{\star}$ denote the (unknown) ground-truth trajectory mode of obstacle~$j$, so that
$\bm{\tau}_{i,t}^{\ast} = \hat{\bm{\tau}}_{i,t}^{(p_1^{\star},\dots,p_{M_t}^{\star})}$.
We introduce an intermediate trajectory $\tilde{\bm{\tau}}_{i,t}^{(p_1^{\star},\dots,p_{M_t}^{\star})}$, which is the trajectory that the \emph{stale} DT would plan if it \emph{magically knew} the ground-truth modes (but still used its stale, compressed states).
Formally,
\begin{equation}
\tilde{\bm{\tau}}_{i,t}^{(p_1^{\star},\dots,p_{M_t}^{\star})}
\triangleq f_{\mathrm{plan}}\bigl(\hat{\bm{z}}_t^{\mathrm{ctx}},\; \tilde{\bm{Y}}_{t}^{(p_1^{\star},\dots,p_{M_t}^{\star})},\; \bm{g}_{i,t}\bigr),
\end{equation}
where $\hat{\bm{z}}_t^{\mathrm{ctx}}$ is the stale context embedding, and $\tilde{\bm{Y}}_{t}^{(p_1^{\star},\dots,p_{M_t}^{\star})}$ uses ground-truth modes for all obstacles.

We decompose the planning gap into two orthogonal components which is given by
\begin{equation}\label{eq:two_term_decomp}
\underbrace{\hat{\bm{\tau}}_{i,t} - \bm{\tau}_{i,t}^{\ast}}_{\text{total gap}}
\;=\;
\underbrace{\bigl(\hat{\bm{\tau}}_{i,t} - \tilde{\bm{\tau}}_{i,t}^{(p_1^{\star},\dots,p_{M_t}^{\star})}\bigr)}_{\text{Term A: pure mode uncertainty}}
\;+\;
\underbrace{\bigl(\tilde{\bm{\tau}}_{i,t}^{(p_1^{\star},\dots,p_{M_t}^{\star})} - \bm{\tau}_{i,t}^{\ast}\bigr)}_{\text{Term B: pure state error}},
\end{equation}
where term~A isolates the effect of not knowing the ground-truth modes, while using exactly the same (stale) state as the DT currently possesses.
Term~B isolates the effect of having stale or compressed states, while using exactly the ground-truth modes.
These two error sources are additive since they operate on different arguments of $f_{\mathrm{plan}}$.

Then, by rearranging \assumpref{assump:planner_separability}, we have
\begin{equation}
\hat{\bm{\tau}}_{i,t} - \tilde{\bm{\tau}}_{i,t}^{(p_1^{\star},\dots,p_{M_t}^{\star})}
\approx \sum_{j=1}^{M_t}\bigl(\hat{\bm{\tau}}_{i,t} - \hat{\bm{\tau}}_{i,t}^{(j,p_j^{\star})}\bigr).
\end{equation}

Taking the squared Euclidean norm and expectation over the random ground-truth modes $\{p_j^{\star}\}$ (drawn independently per \assumpref{assump:obstacle_independence}), we have
\begin{equation}\label{eq:termA_cauchy}
\begin{aligned}
\mathbb{E}\bigl\|\hat{\bm{\tau}}_{i,t} - \tilde{\bm{\tau}}_{i,t}^{(p_1^{\star},\dots,p_{M_t}^{\star})}\bigr\|_2^{2}
&\approx \mathbb{E}\Bigl\|\sum_{j=1}^{M_t}\bigl(\hat{\bm{\tau}}_{i,t} - \hat{\bm{\tau}}_{i,t}^{(j,p_j^{\star})}\bigr)\Bigr\|_2^{2} \\
&\leq M_t\sum_{j=1}^{M_t}\mathbb{E}\bigl\|\hat{\bm{\tau}}_{i,t} - \hat{\bm{\tau}}_{i,t}^{(j,p_j^{\star})}\bigr\|_2^{2},
\end{aligned}
\end{equation}
where the inequality applies Cauchy-Schwarz: $\|\sum_{j=1}^{M}\bm{v}_j\|_2^{2}\leq M\sum_{j=1}^{M}\|\bm{v}_j\|_2^{2}$.

Since $p_j^{\star}\sim\mathrm{Categorical}(\hat{\bm{\pi}}_{j,t})$ independently for each obstacle $j$, $\mathbb{E}\bigl\|\hat{\bm{\tau}}_{i,t} - \hat{\bm{\tau}}_{i,t}^{(j,p_j^{\star})}\bigr\|_2^{2}$ can be further rewritten as
\begin{equation}
\mathbb{E}\bigl\|\hat{\bm{\tau}}_{i,t} - \hat{\bm{\tau}}_{i,t}^{(j,p_j^{\star})}\bigr\|_2^{2}
= \sum_{p=1}^{P}\hat{\pi}_{j,t}^{p}\,\bigl\|\hat{\bm{\tau}}_{i,t} - \hat{\bm{\tau}}_{i,t}^{(j,p)}\bigr\|_2^{2}.
\end{equation}

Summing over all vehicles $i=1,\dots,N$, we obtain the per-obstacle trajectory sensitivity as
\begin{equation}\label{eq:termA_alpha}
\begin{aligned}  
\sum_{i=1}^{N}\mathbb{E}& \bigl\|\hat{\bm{\tau}}_{i,t} - \tilde{\bm{\tau}}_{i,t}^{(p_1^{\star},\dots,p_{M_t}^{\star})}\bigr\|_2^{2} \\
& \leq  M_t\sum_{j=1}^{M_t}\sum_{i=1}^{N}\sum_{p=1}^{P}\hat{\pi}_{j,t}^{p}\,\bigl\|\hat{\bm{\tau}}_{i,t}^{(j,p)} - \hat{\bm{\tau}}_{i,t}\bigr\|^{2}
\triangleq M_t\sum_{j=1}^{M_t}\alpha_{j,t}.
\end{aligned}
\end{equation}

Then, we analyze term~B in (\ref{eq:two_term_decomp}), which measures the planning error induced solely by stale state, assuming ground-truth modes are known.
Both trajectories in term~B share exactly the same trajectory inputs $\tilde{\bm{Y}}_{t}^{(p_1^{\star},\dots,p_{M_t}^{\star})}$; they differ only in the context embedding:
\begin{equation}
\begin{aligned}
\tilde{\bm{\tau}}_{i,t}^{(p_1^{\star},\dots,p_{M_t}^{\star})}
&= f_{\mathrm{plan}}\bigl(\hat{\bm{z}}_t^{\mathrm{ctx}},\; \tilde{\bm{Y}}_{t}^{(p_1^{\star},\dots,p_{M_t}^{\star})},\; \bm{g}_{i,t}\bigr),\\[4pt]
\bm{\tau}_{i,t}^{\ast}
&= f_{\mathrm{plan}}\bigl(\hat{\bm{z}}_t^{\mathrm{ctx},\ast},\; \tilde{\bm{Y}}_{t}^{(p_1^{\star},\dots,p_{M_t}^{\star})},\; \bm{g}_{i,t}\bigr),
\end{aligned}
\end{equation}
where $\hat{\bm{z}}_t^{\mathrm{ctx},\ast}$ is the idealized context embedding.

Based on \assumpref{assump:planner_lipschitz}, the context embedding is computed from the aggregated states of all obstacles via $f_{\mathrm{motion}}$ in~\eqref{eq:dt_traj_pred}.
Since the motion model operates per-obstacle, the squared norm can be decomposed as
\begin{equation}\label{eq:ctx_decomp}
\bigl\|\hat{\bm{z}}_t^{\mathrm{ctx}} - \hat{\bm{z}}_t^{\mathrm{ctx},\ast}\bigr\|_2^{2}
= \sum_{j=1}^{M_t}\bigl\|\hat{\bm{z}}_{j,t}^{\mathrm{agg}} - \hat{\bm{z}}_{j,t}^{\mathrm{agg},\ast}\bigr\|_2^{2}.
\end{equation}

Then, by applying Lemma~\ref{lem:drift} to each obstacle, we have
\begin{equation}\label{eq:apply_lemma}
\bigl\|\hat{\bm{z}}_{j,t}^{\mathrm{agg}} - \hat{\bm{z}}_{j,t}^{\mathrm{agg},\ast}\bigr\|_2^{2}
\leq \eta_{\mathrm{QR}} + \Delta_{j,t}\cdot\delta_{\mathrm{mode}}^{2}.
\end{equation}

Summing over all vehicles $i=1,\dots,N$, we can obtain:
\begin{equation}\label{eq:termB_final}
\begin{aligned} \sum_{i=1}^{N}\bigl\|&\tilde{\bm{\tau}}_{i,t}^{(p_1^{\star},\dots,p_{M_t}^{\star})} - \bm{\tau}_{i,t}^{\ast}\bigr\|_2^{2}
\leq \sum_{i=1}^{N} L_{\mathrm{plan}}^{2}\sum_{j=1}^{M_t}\bigl(2\eta_{\mathrm{QR}} \!\! + \!\! 2\Delta_{j,t}^2\!\!\cdot\!\delta_{\mathrm{mode}}^{2}\bigr)\\[4pt]
&= N L_{\mathrm{plan}}^{2}\sum_{j=1}^{M_t}\bigl(2\eta_{\mathrm{QR}} + 2\Delta_{j,t}^2\cdot\delta_{\mathrm{mode}}^{2}\bigr)
\triangleq  \sum_{j=1}^{M_t}\varepsilon(\Delta_{j,t}),
\end{aligned}
\end{equation}
where $\varepsilon(\Delta_{j,t}) \;\triangleq\; N L_{\mathrm{plan}}^{2}\bigl(2\eta_{\mathrm{QR}} + 2\Delta_{j,t}^2\cdot\delta_{\mathrm{mode}}^{2}\bigr)$.

Finally, based on ~\eqref{eq:two_term_decomp} and the triangle inequality, we have
\begin{equation}
\bigl\|\hat{\bm{\tau}}_{i,t} - \bm{\tau}_{i,t}^{\ast}\bigr\|_2^{2}
\leq \bigl\|\hat{\bm{\tau}}_{i,t} - \tilde{\bm{\tau}}_{i,t}^{(p_1^{\star},\dots,p_{M_t}^{\star})}\bigr\|_2^{2}
\;+\;
\bigl\|\tilde{\bm{\tau}}_{i,t}^{(p_1^{\star},\dots,p_{M_t}^{\star})} - \bm{\tau}_{i,t}^{\ast}\bigr\|_2^{2}.
\end{equation}
Then, by applying the bounds from \eqref{eq:termA_alpha} and \eqref{eq:termB_final}, we have
\begin{equation}\label{eq:final_assembly}
\sum_{i=1}^{N}\mathbb{E}\bigl\|\hat{\bm{\tau}}_{i,t}-\bm{\tau}_{i,t}^{\ast}\bigr\|_2^{2}\leq \sum_{j=1}^{M_t}\bigl[\alpha_{j,t} + \varepsilon(\Delta_{j,t})\bigr].
\end{equation}
This completes the proof.
\end{proof}

\end{document}